%% file: iclr2026_conference.tex
\documentclass{article} 
\usepackage{iclr2026_conference,times}

\usepackage{hyperref}
\usepackage{url}
\usepackage{amsmath}
\usepackage{amssymb}
\usepackage{mathtools}
\usepackage{amsthm}

\usepackage{graphicx}
\usepackage{subfigure}
\usepackage{arydshln}
\usepackage{enumerate}

\input{math.tex}

\usepackage[capitalize,noabbrev]{cleveref}
\usepackage{wrapfig,lipsum,booktabs}
\usepackage[most]{tcolorbox}
\usepackage{xcolor}

\newtheorem{lemma}{Lemma}
\newtheorem{definition}{Definition}
\newtheorem{proposition}{Proposition}
\newtheorem{thm}{Theorem}
\newtheorem{assum}{Assumption}

\usepackage[textsize=tiny]{todonotes}

\newcommand{\bfsection}[1]{\noindent\textbf{#1}}
\def\HS{\hspace{\fontdimen2\font}}

\newtcolorbox{examplebox}{
    colback=gray!12,      
    colframe=gray!12,     
    boxrule=0pt,
    arc=0pt,              
    left=4pt,
    right=4pt,
    top=4pt,
    bottom=4pt,
    boxsep=0pt,
    before skip=6pt,
    after skip=6pt,
    breakable
}

\title{Towards Identifiable Representations under Misspecified Structure}

\author{
Yuke Li$^{1}$\thanks{Equal contribution.}, Yujia Zheng$^{2 *}$, Ziyi Chen$^{1}$, Kun Zhang$^{2}$, Heng Huang$^{1}$\\
$^{1}$ University of Maryland College Park, College Park, MD, USA
\\
$^{2}$ Carnegie Mellon University, Pittsburgh PA, USA \\
}

\iclrfinalcopy 

\begin{document}

\maketitle

\begin{abstract}

The presence of noise that depends on the latent variables poses a fundamental challenge to identifiability. Existing results rely on conditional independence among the observations given the latent variables. 
We study a more general \emph{misspecified structure}, where this conditional factorization does not hold, and establish both precise and approximate identifiability guarantees.
We characterize structural misspecification as a perturbed factor analysis problem.
For precise identifiability, we establish subspace identifiability under spectral separation and controlled perturbation, followed by component-wise identifiability under structural sparsity. 
When the precise condition is not guaranteed, we derive an approximate subspace-identifiability theorem.
Based on these results, we develop an unsupervised variational estimator for recovering latent variables. Experiments demonstrate the effectiveness of the proposed framework.

\end{abstract}

\input{1_intro}

\input{2_setup}

\input{3_disident_neurips25}

\input{4_variden}
\input{5_est}

\input{6_exp}



\section*{AI Use Statement}

In this work, we used generative AI tools solely to assist with language editing, including improving grammar, clarity, conciseness, and overall wording of the manuscript. Generative AI tools were not used to generate the research ideas, theoretical results, experimental design, or scientific conclusions. All AI-assisted text was reviewed and revised by the authors, who take full responsibility for the final content of this work.

\newpage
\bibliography{ref}
\bibliographystyle{iclr2026_conference}

\appendix

\input{7_appendix}


\end{document}

%% file: math.tex
\usepackage{amsmath,amsfonts,bm}

\def\eqref#1{equation~\ref{#1}}

\def\1{\bm{1}}

\def\rva{{\mathbf{a}}}
\def\rvb{{\mathbf{b}}}
\def\rvc{{\mathbf{c}}}
\def\rvd{{\mathbf{d}}}

\def\rvn{{\mathbf{n}}}

\def\rvu{{\mathbf{u}}}
\def\rvv{{\mathbf{v}}}
\def\rvw{{\mathbf{w}}}
\def\rvx{{\mathbf{x}}}
\def\rvy{{\mathbf{y}}}
\def\rvz{{\mathbf{z}}}

\DeclareMathAlphabet{\mathsfit}{\encodingdefault}{\sfdefault}{m}{sl}
\SetMathAlphabet{\mathsfit}{bold}{\encodingdefault}{\sfdefault}{bx}{n}



%% file: 1_intro.tex
\section{Introduction}\label{sec:intro}

Learning meaningful representations from the data generating process with structured noise remains an open challenge. Several works treats the noise terms as known auxiliaries \citep{lachapelle2024additive,cauca_neurips23,sparsecausal_icml23,cauca_neurips23,causal_uai23,zheng2022identifiability,yaomulti,lachapelle2024nonparametric,li2025identification,song2024temporally,rajendran2024from,sparse_icml24,brady2025interaction}.
One might argue that noise term can simply be absorbed into an expanded latent space.
The approaches of \citep{kugelgen2021selfsupervised,causal_icml22,causal_iclr23,timilsina2026contentstyle,xu2026identifying} pursue this route but require partitioning them into invariant / variant components across environments.
\citet{sun2025causal} assumes independence between noise and latent variables, while \citep{kong2023identification} requires that the relations between a pair of latent variables have to be sufficiently distinct.

When the noise itself depends on the latent variables, which we refer to as \emph{dependent noise}, establishing identifiability becomes substantially more challenging. Prior work \citep{hu2008instrumental,zheng2025nonparametric} addresses this setting under the assumption that the observations become conditionally independent given the latent variables. A few recent works extends this assumption to temporal data \citep{li2026towards,li2026online,fu2026learning}. However, this assumption can be violated in realistic systems. 
As illustrated in Fig.~\ref{fig:mis}, in chest X-ray diagnosis, the latent variable corresponds to the patient's underlying lung-cancer state, the observation is the chest X-ray image, and the noise can be represented by the patient's inspiration level.
The inspiration level may itself depend on the disease state, while both the disease state and the inspiration level affect multiple regions of the observed image. Consequently, different image regions can remain statistically dependent even after conditioning on the lung-cancer state. We refer to this setting as {\bf\emph{misspecified structure}}, and study how to recover the latent variables under such misspecification.

We address this challenge by characterizing structural misspecification as a perturbed factor analysis problem. Equation~\ref{eq:bound} quantifies the resulting perturbation and provides the basis for two complementary identifiability regimes. First, when the perturbation is sufficiently small relative to the spectral separation, Theorem~\ref{thm:submanifold_iden} establishes precise subspace identifiability, while Proposition~\ref{thm:variabe_iden} further reduces the remaining ambiguity to component-wise identifiability under structural sparsity. Second, when the sufficient condition for precise identifiability is no longer guaranteed, Theorem~\ref{thm:approx_vareps} establishes approximate subspace identifiability by bounding the identification error between the recovered representation and its admissible equivalence class. 
These theoretical contributions yields, to the best of our knowledge, one of the first general frameworks for uncovering latent variables in misspecified structure with proper identifiability guarantees.

Leveraging these theoretical insights, we propose an unsupervised method, which utilizes a variational inference-based learning objective specifically designed to uncover latent variables. 
The same learning objective is used in both the precise and approximate regimes. Experiments on synthetic and real-world data demonstrate improved latent recovery under a misspecified structure. In particular, we evaluate precise component-wise identification and approximate subspace identification separately, supporting the two theoretical regimes developed in this work.

%% file: 2_setup.tex
\section{Problem Setting}\label{sec:setup}



Let $\rvx\in\mathbb{R}^K$ denote the $K$ dimensional observation, $\rvz\in\mathbb{R}^N$ denote the latent variable.
Our data generating process is formulated by:
\begin{align}\label{eq:dgp}
    \rvx = g(\rvz,\epsilon), 
    \quad
    \epsilon = e(\rvz,\eta)
\end{align}
We assume 
$g$ to be nonlinear, nonparametric and smooth functions. 
Importantly, $e$ denotes another nonlinear, nonparametric and smooth function, allowing the $\epsilon$ to be dependent on $\rvz$. 
$\eta$ denotes an exogenous variable sampled from $\mathcal{N}(0,I)$. 

Our primary goal is to identify latent variables $\rvz$ from observed data $\rvx$. 
Suppose there exists an invertible and differentiable transformation $h$, we further define subspace identifiability and component-wise identifiability as follows:
\begin{definition}\label{def:subspace_iden}
    \bfsection{(Subspace Identifiability)}
    The mapping $h$ is said to achieve subspace identifiability if it is invertible and the following transformation holds: $\hat{\rvz} = h(\rvz)$, where $\hat{\rvz}$ denotes the estimation of $\rvz$.
\end{definition}

\begin{definition}\label{def:component_iden}
    \bfsection{(Component-wise Identifiability)}
    For an individual component of the latent variable $\rvz^n$ $(\rvn\in[1,N])$, there exists a unique component $\hat{n}$ $(\hat{n}\in [1,N])$ of $\hat{\rvz}$ matches $\rvz^n$ up to a permutation $\pi$, such that $\hat{\rvz}^{\hat{n}}=h^{\hat{n}}(\rvz^{\pi(n)})$.
    Then $\rvz^n$ is component-wise identifiable.
\end{definition}

In this work, 
we breakdown the identifiability problem in two steps. First, we provide our findings on subspace identifiability in Section~\ref{sec:preciseiden} and ~\ref{sec:error_bound}. Subsequently, we present the component-wise identifiability result in Section~\ref{sec:variden}.

\bfsection{Remark on Misspecified Structure}

\input{figs/fig_intro}

Our main challenge arises from relaxing the conditional-independence structure commonly used in prior identification results \citep{hu2008instrumental,zheng2025nonparametric,fu2026learning,
li2026online,li2026towards}.
Following their multi-partition approach, we partition the observation as three disjoint blocks $\rvx=\{\rvx_\rva,\rvx_\rvb,\rvx_\rvc\}$ and the structural misspecification by:

\begin{definition}\label{def:misspeci}\bfsection{(Misspecified Structure)}
    The conditional structure is misspecified when the three observed blocks do not factorize conditionally on $\rvz$, i.e., $p(\rvx|\rvz) \neq p(\rvx_\rva|\rvz)p(\rvx_\rvb|\rvz)p(\rvx_\rvc|\rvz)$
\end{definition}

Figure~\ref{fig:mis} illustrates the misspecified structure using the lung-cancer example introduced in Section~\ref{sec:intro}. 
The existing results based on conditional independence rely on a factorization of the observable distribution that no longer holds under Definition~\ref{def:misspeci}. 
Section~\ref{sec:disiden} develops identifiability guarantees that remain valid under such misspecifications.

%% file: figs/fig_intro.tex
\begin{wrapfigure}{r}{5.5cm}
\vspace{-40px}
\includegraphics[width=\linewidth]{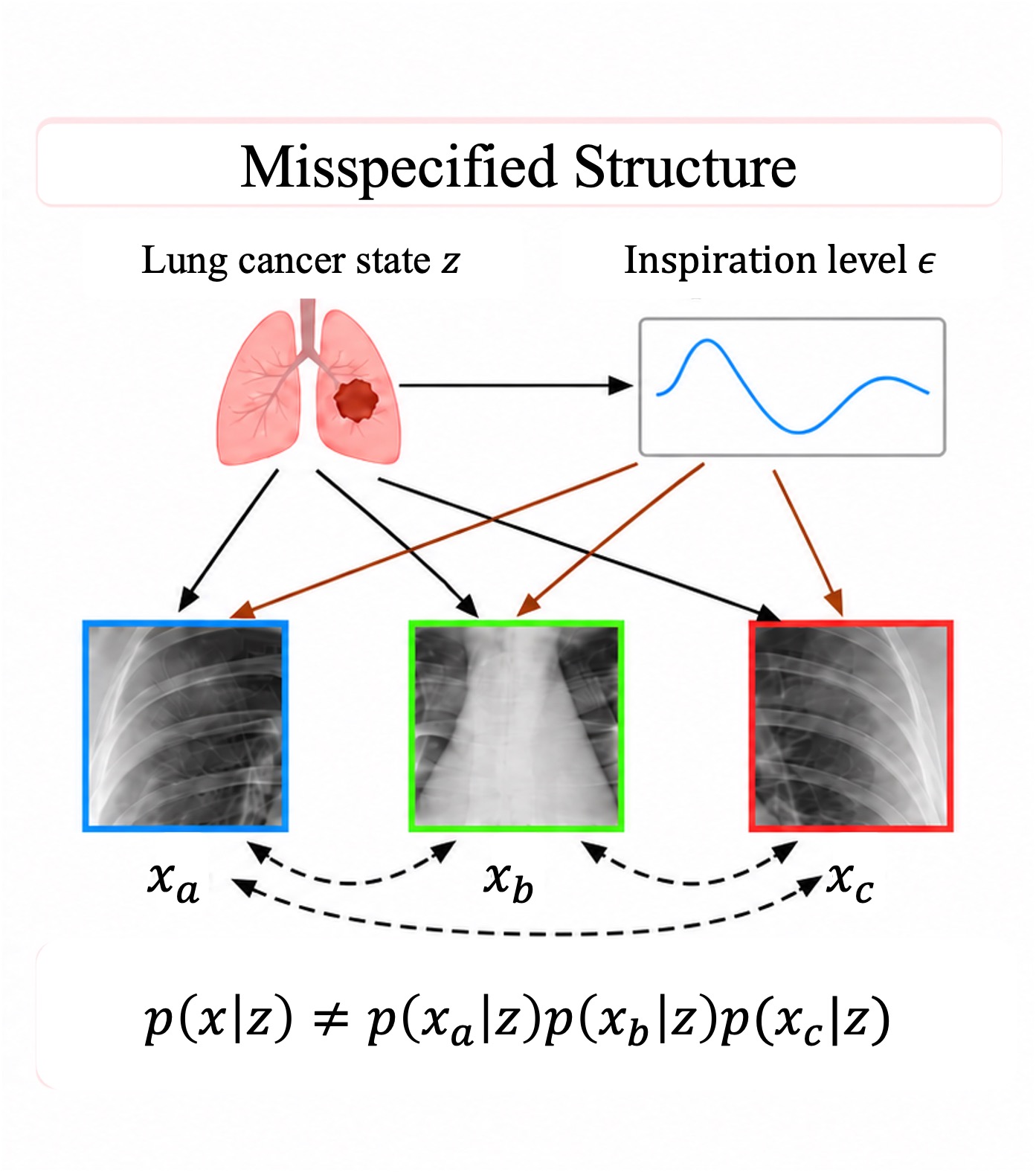}
\vspace{-25px}
\caption{
Visualization of the data generations of Eq.~\ref{eq:dgp} under misspecified structure.
}
\label{fig:mis}
\end{wrapfigure}

%% file: 3_disident_neurips25.tex
\section{Identifiability Results}\label{sec:disiden}

We first analyze the misspecified structure through factor analysis. To this end, we 
introduce two types of operators to characterize the structural
misspecification: 
\begin{definition}\bfsection{(Integral Linear Operator)}\label{def:integral_operator}
Consider random variables $\rvx_*$ and $\rvz$ with supports $\mathcal{X}_*$ and $\mathcal{Z}$ respectively. The integral linear operator $L_{\rvx_*|\rvz}$ maps a function $f \in F(\mathcal{Z})$ to another function $L_{\rvx_*|\rvz}f \in F(\mathcal{X}_*)$, defined by:
\begin{equation}\label{eq:operator}
    (L_{\rvx_*|\rvz}f)(\rvx_*) = \int_{\mathcal{Z}} p(\rvx_*|\rvz)f(\rvz)\,d\mathcal{Z}, \quad \forall \rvx_* \in \mathcal{X}_*
\end{equation}
where $p(\rvx_*|\rvz)$ denotes the conditional density of $\rvx_*$ given $\rvz$.
\end{definition}

Also, we present the definition of multiplication operator $\Lambda$:
\begin{definition}\bfsection{(Multiplication Operator)}\label{def:multi_operator}
\it
Let $r$ and $v$ be random variables with supports $\mathcal R$ and $\mathcal V$, respectively. For each fixed $r\in\mathcal R$, define the multiplication operator
$\Lambda_{r\mid v}: (\Lambda_{r\mid v}f)(v) := p(r\mid v)f(v)$, where $f\in\mathcal L^2(\mathcal V)$.
\end{definition}

With Definition~\ref{def:integral_operator} and ~\ref{def:multi_operator} established, we impose the following Assumption~\ref{assum:error} on $p(\rvx|\rvz)$.

\begin{assum}\label{assum:error}
The joint distribution of $(\rvx, \rvz)$ admits a bounded density with respect to a suitable product measure defined on their supports. Furthermore, all marginal and conditional densities derived from this joint distribution are also bounded.
\end{assum}

We begin by decomposing the bounded operator $L_{\rvx_\rva, \rvx_\rvb|\rvx_\rvc}$. The misspecified structure in Definition~\ref{def:misspeci} renders the previous decomposition and identifiability results from \citet{hu2008instrumental,zheng2025nonparametric}
$L_{\rvx_\rva,\rvx_\rvb|\rvx_\rvc} = L_{\rvx_\rva|\rvz}\Lambda_{\rvx_\rvb|\rvz}L_{\rvz|\rvx_\rvc}$ ( Section~\ref{sec:preliminaries})
invalid. 
To account for this violation, we explicitly define the difference operator $\rvd(\rvx_\rvb)$ as follows:
\begin{align}\label{eq:error_oper_1}
    & \rvd(\rvx_\rvb) = L_{\rvx_\rva,\rvx_\rvb|\rvx_\rvc} - L_{\rvx_\rva|\rvz}\Lambda_{\rvx_\rvb|\rvz}L_{\rvz|\rvx_\rvc} \nonumber \\
    & \rvd = \int \rvd(\rvx_\rvb) \, d\rvx_\rvb = L_{\rvx_\rva|\rvx_\rvc} - L_{\rvx_\rva|\rvz}L_{\rvz|\rvx_\rvc}
\end{align}
where $\Lambda_{\rvx_\rvb|\rvz}$ satisfies $\int_{\rvx_\rvb}\Lambda_{\rvx_\rvb|\rvz}\,d\rvx_\rvb = I$, with $I$ representing the identity operator.
Therefore,
$L_{\rvx_\rva|\rvz}L_{\rvz|\rvx_\rvc} = L_{\rvx_\rva|\rvx_\rvc} - \rvd$.

To the end of further decomposition, we also assume:
\begin{assum}\label{assum:injective}
The operators $L_{\rvx_\rva|\rvz}$
and $L_{\rvx_\rva|\rvx_\rvc}$ are injective and bounded-below.
\end{assum}

Assumption~\ref{assum:injective} implies that $L^{-1}_{\rvx_\rva|\rvz}$
and $L^{-1}_{\rvx_\rva|\rvx_\rvc}$ are well-defined and bounded.
Leveraging all previous assumptions, we are able to derive the following:
\begin{align}\label{eq:error_derivation}
    L_{\rvx_\rva|\rvz}\Lambda_{\rvx_\rvb|\rvz}L^{-1}_{\rvx_\rva|\rvz}
    = (L_{\rvx_\rva,\rvx_\rvb|\rvx_\rvc} - \rvd(\rvx_\rvb))(L_{\rvx_\rva|\rvx_\rvc} - \rvd)^{-1} = L_{\rvx_\rva,\rvx_\rvb|\rvx_\rvc}L^{-1}_{\rvx_\rva|\rvx_\rvc} + \it{Per}
\end{align}
where $\it{Per}$ denotes the perturbation operator that represents the structural misspecification.
Rather than requiring $\mathit{Per}$ itself to be known, we characterize its upper bound through $||Per||_{op}\leq \overline{Per}$, which is derived in Lemma~\ref{lemma:Per} in Section~\ref{sec:bound_per}. 
Consequently, our aim becomes to identify $\rvz$ from $L_{\rvx_\rva|\rvz}\Lambda_{\rvx_\rvb|\rvz}L^{-1}_{\rvx_\rva|\rvz}$ of Eq.~\ref{eq:error_derivation}.

It is worth noting that \citet{zheng2025nonparametric} particularly proves the subspace identifiability for $\it{Per}=0$ when the conditional independence holds. However, Definition~\ref{def:misspeci} indicates that $\it{Per}\neq 0$, and \citep{zheng2025nonparametric} can be seen as a special case of this work. 

Before going further, we must address two critical challenges of $L_{\rvx_\rva|\rvz}\Lambda_{\rvx_\rvb|\rvz}L^{-1}_{\rvx_\rva|\rvz}$: scaling ambiguity and spectral value degeneracy.
The scaling ambiguity arises when there exists a non-zero constant $c$ such that $\int c p(\rvx_\rva|\rvz)\,d\rvx_\rva > 1$. However, this ambiguity are solved since the probability density constraint $\int p(\rvx_\rva|\rvz)\, d\rvx_\rva = 1$ must always hold.
The degeneracy occurs if spectral value are repeated, thereby associating multiple eigenfunctions to a single spectral value. To eliminate this degeneracy, we introduce the following additional assumption:
\begin{assum}\label{assum:distinct_eigenvalue}
The operator
$L_{\rvx_\rva,\rvx_\rvb|\rvx_\rvc}L^{-1}_{\rvx_\rva|\rvx_\rvc}$ has distinct spectral values with cardinality equal to that of $\Lambda_{\rvx_\rvb|\rvz}$. Also, 
The map
$
\Theta:\mathcal Z\rightarrow\mathcal D, s.t.
\Theta(\rvz)=p(\rvx_\rvb\mid\rvz)
$
is injective.
\end{assum}

Under misspecification, Eq.~\ref{eq:error_derivation} shows that the $L_{\rvx_\rva|\rvz}\Lambda_{\rvx_\rvb|\rvz}L^{-1}_{\rvx_\rva|\rvz}$and $L_{\rvx_\rva,\rvx_\rvb|\rvx_\rvc}L^{-1}_{\rvx_\rva|\rvx_\rvc}$ differ by $\mathit{Per}$. We therefore quantify how far the spectral values of $L_{\rvx_\rva,\rvx_\rvb|\rvx_\rvc}L^{-1}_{\rvx_\rva|\rvx_\rvc}$ may move under this perturbation. 
To this end, we need to first justify the spectral analysis of Eq.~\ref{eq:error_derivation} by establishing that $L_{\rvx_\rva|\rvz}\Lambda_{\rvx_\rvb|\rvz}L^{-1}_{\rvx_\rva|\rvz}$ admits a unique spectral resolution. 
Since $p(\rvx_\rvb|\rvz)$ is real-valued, $\Lambda_{\rvx_\rvb\mid\rvz}$ is self-adjoint, which already admits a unique projection-valued measure. 
Its similarity transformation through $L_{\rvx_\rva|\rvz}$ then induces the corresponding unique spectral resolution of
$L_{\rvx_\rva|\rvz}\Lambda_{\rvx_\rvb|\rvz}
L^{-1}_{\rvx_\rva|\rvz}$ \citep{dunford1971linear}.
Thus, $L_{\rvx_\rva|\rvz}\Lambda_{\rvx_\rvb|\rvz}L^{-1}_{\rvx_\rva|\rvz}$ is equivalently the spectral resolution of $L_{\rvx_\rva,\rvx_\rvb|\rvx_\rvc}L^{-1}_{\rvx_\rva|\rvx_\rvc}+\mathit{Per}$.
The detailed proof is provided in Appendix~\ref{proof:prop1}.

Let $\rho\in \sigma(L_{\rvx_\rva,\rvx_\rvb|\rvx_\rvc}L^{-1}_{\rvx_\rva|\rvx_\rvc})$ and $\rho_\Lambda\in\sigma(L_{\rvx_\rva|\rvz}\Lambda_{\rvx_\rvb|\rvz}L^{-1}_{\rvx_\rva|\rvz})$, where $\sigma(*)$ denotes the spectral value.
The resolvent perturbation bound for bounded operator from \citep{trefethen2005spectra,amor2003extension} yields:
\begin{align}\label{eq:bound}
    dist(\rho,\rho_\Lambda)\leq 
    \kappa\overline{Per}
\end{align}
where $dist(\rho,\rho_\Lambda)=\inf_{\rho_\Lambda}|\rho-\rho_\Lambda|$, $\kappa=||L_{\rvx_\rva|\rvz}||_{op}||L^{-1}_{\rvx_\rva|\rvz}||_{op}$, and $||*||_{op}$ denotes the operator norm. 
Notably, $\kappa\geq 1$, since$||L_{\rvx_\rva|\rvz}L^{-1}_{\rvx_\rva|\rvz}||_{op}=||I||_{op}\leq ||L_{\rvx_\rva|\rvz}||_{op}||L^{-1}_{\rvx_\rva|\rvz}||_{op}$, where $I$ denotes the identity operator.
Let us define $\eta=\inf_{i\neq j}\frac{|\varrho^i-\varrho^j|}{2\kappa}-\alpha\geq 0$ for some constant $\alpha>0$, where $\varrho$ denotes the non-empty spectrum, $i,j$ denote the indices of the spectrum. 
Similarly, let $\varrho_\Lambda$ denotes the spectrum of $\rho_\Lambda$. We opt to measure $dist(\rho,\rho_\Lambda)=|\rho^i-\rho^i_\Lambda|$.
Equation~\ref{eq:bound} forms the basis of the precise and approximate identifiability results developed next.

\begin{examplebox}\bfsection{Assumption Intuitions}

$\bullet$ Assumption~\ref{assum:error} is a regularity condition to ensure that the operator expressions below are well defined.

$\bullet$ Assumption~\ref{assum:injective} is mild and prevents these operators from collapsing distinct functions to the same output.

$\bullet$ Assumption~\ref{assum:distinct_eigenvalue} prevents different latent variables from being indistinguishable through the same spectral value. Injectivity of $\Theta$ further ensures that distinct latent states induce distinct conditional distributions of $\rvx_\rvb$.
\end{examplebox}

\subsection{Precise subspace Identifiability}\label{sec:preciseiden}

After facilitating the structural misspecification in Eq.~\ref{eq:bound}, we now establish the precise identifiability result that meets Definition~\ref{def:subspace_iden}.
The key requirement is that $\overline{Per}$ needs to be sufficiently small so that distinct spectral values remain distinguishable.

Consider $||Per||_{op}\leq\overline{\it{Per}}<\eta$ 
\footnote{We provide the sufficient condition of $\overline{Per}<\eta$ in Lemma~\ref{lemma:sufficient_eta} in Section~\ref{sec:bound_per}.}
for some non-zero constant $\eta$, and substitute it into Eq.~\ref{eq:error_derivation} resulting in:
\begin{align}\label{eq:inequ}
    L_{\rvx_\rva|\rvz}\Lambda_{\rvx_\rvb|\rvz}L^{-1}_{\rvx_\rva|\rvz} < L_{\rvx_\rva,\rvx_\rvb|\rvx_\rvc}L^{-1}_{\rvx_\rva|\rvx_\rvc}+\eta
\end{align}

Given Eq.~\ref{eq:inequ},
we propose the following for the distinctness of $\rho^i_\Lambda$: 

\begin{proposition}\label{prop:eigenvalue}
Suppose Assumptions~\ref{assum:error} $\sim$ \ref{assum:distinct_eigenvalue} hold, and if both $\varrho^i$ and $\varrho^i_\Lambda$ include only $\rho^i$ and $\rho_\Lambda^i$, respectively. Then, for all $i>1$ and for all $\rvz$,
$$
\rho^i_{\Lambda}\in \left(\max\{\rho^i-\kappa\overline{\it{Per}}, \rho^{i-1}+\kappa\overline{\it{Per}}\}, \rho^i+\kappa\overline{\it{Per}}\right)
$$
and if $i=1$, for all $\rvz$,
$$
\rho^i_{\Lambda}\in (\rho^i-\kappa\overline{\it{Per}}, \rho^i+\kappa\overline{\it{Per}})
$$
guarantee the distinctness of $\rho_\Lambda$.
\end{proposition}

\begin{examplebox}
\bfsection{Insight:} 
For fixed $\rvx_\rvb$, $\sigma(\Lambda_{\rvx_\rvb|\rvz})$ is determined by the essential range of $p(\rvx_\rvb|\rvz)$.
Proposition~\ref{prop:eigenvalue} ensures that $\overline{Per}$ does not collapse two distinct $\rho_\Lambda$ into the same regime. Together with the injectivity of $\Theta$ in Assumption~\ref{assum:distinct_eigenvalue}, $\rho_\Lambda$ remains invariant by re-labeling through permuting $\rvz$. This can be formally described by $\rvz'=h'(\rvz)$, where $h'$ denotes a bijection $h':\mathcal{Z}\rightarrow\mathcal{Z}$, and $\rvz'$ is a permuted version of $\rvz$.
\end{examplebox}

\textbf{Proof Sketch:} We proceed by contradiction. Assuming the existence of distinct spectral values $\rho^1,\rho^2$ satisfy $|\rho^1-\rho^i_\Lambda|\leq \kappa\overline{\it{Per}}$ and $|\rho^2-\rho^i_\Lambda|\leq \kappa\overline{\it{Per}}$. This leads directly to $\overline{\it{Per}}\geq\eta$, contradicting the condition $\overline{\it{Per}}<\eta$. 

Notably, the bijection $h'$ introduced above might not coincide with $h$ defined in Definition~\ref{def:subspace_iden}. 
To meet the requirement of $h$ being differentiable, we introduce the following assumption:

\begin{assum}\label{assum:M}
There exists a map $M: \mathcal D\rightarrow\mathcal Z$  such that $M(p(\rvx_\rvb|\rvz))=M(p(\rvx_\rvb|h'(\rvz)))=h(\rvz)$, where $\mathcal D$ is the support of $p(\rvx_\rvb|\rvz)$, $h$ is a differentiable transformation.
\end{assum}

We are now ready to state our main theoretical result on subspace identifiability:

\begin{thm}\label{thm:submanifold_iden}
Consider observed variables $\rvx \in \mathbb{R}^K$ and the estimated latent variables $\hat{\rvz}\in\mathbb{R}^N$, suppose that Proposition~\ref{prop:eigenvalue} and Assumption~\ref{assum:M} hold, $\exists h, s.t., \hat{\rvz}=h(\rvz)=h'(\rvz)$, 
In other words, $\rvz$ must be subspace identified.
\end{thm}

\begin{examplebox}\bfsection{Insight:}
    Theorem~\ref{thm:submanifold_iden} characterizes a perturbation regime in which the subspace identifiability in Definition~\ref{def:subspace_iden} survives from structural misspecification. 
\end{examplebox}

The detailed proofs of both Proposition~\ref{prop:eigenvalue} and Theorem~\ref{thm:submanifold_iden} are in Sec.\ref{proof:prop1}.

\begin{examplebox}
\bfsection{Assumption intuition:}

    Assumption~\ref{assum:M} ensures that the transformation $h$ established through the Proposition~\ref{prop:eigenvalue} belongs to the transformation class required by Definition 1.
\end{examplebox}

\subsection{Approximate Subspace Identifiability}\label{sec:error_bound}

Theorem~\ref{thm:submanifold_iden} establishes precise identifiability under the perturbation condition $\overline{Per}<\eta$. However, this inequality is not guaranteed when the sufficient condition in Lemma~\ref{lemma:sufficient_eta} is not satisfied.
Consequently, Theorem~\ref{thm:submanifold_iden} may no longer hold even under Assumptions~\ref{assum:error}-\ref{assum:distinct_eigenvalue}.
Motivated by this limitation, we consider \emph{approximate identifiability}, which quantifies the deviation of the recovered representation from its admissible equivalence class:

\begin{definition}\label{def:approx_iden}
\bfsection{(Approximate identifiability)}
Let $\mathcal H$ denote the class of admissible invertible and smooth transformations in Definition~\ref{def:subspace_iden}. The latent variable $\rvz$ is $\varepsilon$-approximately identifiable with respect to $\mathcal H$ if
\begin{align}
d_{\mathcal H}(\hat{\rvz},\rvz):=\inf_{h\in\mathcal H}|\hat{\rvz}-h(\rvz)|= \varepsilon
\label{eq:approx_iden_def}
\end{align}
where the identifiability error $\varepsilon\geq 0$. In particular, $\varepsilon=0$ reduces to precisely subspace identifiability.
\end{definition}

We are now ready to connect $\varepsilon$ with $\overline{Per}$. To this end, we impose the following assumption and Theorem:

\begin{assum}\label{assum:smooth_ztopb}
The latent space $\mathcal Z$ is compact, and the conditional-distribution map $\Theta$ in Assumption~\ref{assum:subiden_car_supp} is continuously differentiable. Moreover, there exists a constant $c_{\Theta}>0$ such that, for an admissible invertible transformation $h\in\mathcal H$, $\Theta$ satisfies the lower bound of Bi-Lipschitz continuity by:
\begin{align}
c_{\Theta}|\hat{\mathbf{z}}-h(\mathbf{z})| \leq \|\Theta(\mathbf{z})-\widehat{\Theta}(\hat{\mathbf{z}})\|_{\mathcal F}.
\label{eq:inverse_lipschitz_thm2}
\end{align}
\end{assum}

\begin{thm}\label{thm:approx_vareps}
Suppose Assumptions~\ref{assum:error}--\ref{assum:distinct_eigenvalue} and Assumption~\ref{assum:smooth_ztopb} hold, and the spectral perturbation bounds in Eq.~\ref{eq:bound} apply to the corresponding spectral values of the true and estimated representations. Then the approximate identification error in Definition~\ref{def:approx_iden} satisfies
\begin{align}
\varepsilon\leq\frac{\kappa\overline{Per}+\hat{\kappa}\widehat{\overline{Per}}}{c_{\Theta}}
\label{eq:approx_vareps_bound}
\end{align}
\end{thm}

\begin{examplebox}
\bfsection{Insight:}
Theorems~\ref{thm:submanifold_iden} and \ref{thm:approx_vareps} characterize two identifiability regimes.
When $\overline{\mathit{Per}}<\eta$, the spectral value remain separated and precise subspace identifiability is guaranteed.
When this condition cannot be guaranteed, Theorem~\ref{thm:approx_vareps} provides a relaxation: 
there exists an error $\varepsilon$ bounded by Eq.~\ref{eq:approx_vareps_bound}.
\end{examplebox}

Under Assumptions~\ref{assum:error}--\ref{assum:distinct_eigenvalue}, the same spectral perturbation argument as in Eq.~\ref{eq:bound} yields
$|\rho^i-\rho^i_{\Lambda}|\leq\kappa\overline{Per}$ and $|\rho^i-\hat{\rho}^i_{\Lambda}|\leq\hat{\kappa}\widehat{\overline{Per}}$.
Therefore, the triangle inequality results in
$|\rho^i_{\Lambda}-\hat{\rho}^i_{\Lambda}|\leq\kappa\overline{Per}+\hat{\kappa}\widehat{\overline{Per}}$.
Combining this perturbation bound with the inverse-Lipschitz condition in Assumption~\ref{assum:smooth_ztopb} yields the upper bound in
Eq.~\ref{eq:approx_vareps_bound}. The detailed proof is provided in Sec.~\ref{sec:approxiden}.

\begin{examplebox}\bfsection{Assumption intuition:}

    Assumption~\ref{assum:smooth_ztopb} provides the quantitative measure to connect the discrepancy in $\mathcal D$ with discrepancy in $\mathcal Z$.
\end{examplebox}

%% file: 4_variden.tex
\subsection{Precise Component-wise Identifiability}\label{sec:variden}

In this section, we introduce our precise component-wise identifiability upon Theorem~\ref{thm:submanifold_iden}. {\it Extending the result to approximate component-wise identifiability requires additional component-level information and is left for future work.}

Theorem~\ref{thm:submanifold_iden} guarantees that $\hat{\rvz}=h(\rvz)$ is not a function of $\epsilon$, hence $\rvz$ and $\epsilon$ are disentangled. 
We now focus on identifying $\rvz$ in a component-wise manner.
Consider that each dimensions of $\rvz$ are independent of each other ($p(\rvz)=\prod_{n=1}^N p(\rvz^n)$), 
Taking inspirations from \citep{sparsecausal_icml23,sparsecausal_neurips23_1,sparse_icml24},
we extend the definition of structural sparsity to the function from $\rvz$ to $p(\rvx|\rvz)$ and establish the following proposition:

\begin{proposition}\label{thm:variabe_iden}
Suppose (Theorem~\ref{thm:submanifold_iden}) holds, and the following assumptions and regularization conditions are satisfied:
\begin{enumerate}
    \item\label{assum:independence_z} {\bf (Independence)}: For $n\in [N]$, $\rvz^n$ are independent of each other: $p(\rvz)=\prod_{n=1}^N p(\rvz^n)$. 
    \item\label{assum:ssparse} {\bf (Structural sparsity)}: The map $\Phi: \mathcal{Z}\rightarrow \mathcal P$ is smooth, where $\mathcal{Z}$ denotes the support of $\rvz$, and $\mathcal{P}$ denotes the support of $p(\rvx|\rvz)$ ( $\mathcal P \subset \mathcal{L}^2(\mathcal{X},\mu)$ ($\mathcal X$ denotes the space of $\rvx$, $\mu$ denotes the measure)). Let $G$ and $\hat{G}$ be the binary support matrix of Jacobian $J_{\Phi}(\rvz)$ and $J_{\hat{\Phi}}(\hat{\rvz})$, respectively, 
i.e., $G^{i,j}= 1$
 if $||\frac{\partial p(\rvx^i|\rvz)}{\partial{\rvz^j}}||_{\mathcal{L}^2}\neq 0$,
    and $0$ for otherwise.
    For each $n \in [1,N]$, there exists a subset of indices $\mathcal{C}_k$ satisfying $\bigcap_{m\in\mathcal{C}_k}G^{m,:}=\{n\}$;
    \item\label{assum:sparsity} $|\hat{G}|_0\leq |G|_0$, where $|*|_0$ is the $\ell_0$ norm.
\end{enumerate}
Then, $\hat{\rvz}$ must correspond component-wise to a permutation of the true latent variables $\rvz$.
\end{proposition}

\begin{examplebox}
\bfsection{Insight:}
Theorem~\ref{thm:submanifold_iden} first reduces the ambiguity from an arbitrary latent representation to an invertible transformation $\hat{\rvz}=h(\rvz)$. Proposition~\ref{thm:variabe_iden} further uses structural sparsity to rule out transformations that mix multiple latent coordinates.
Consequently, even under the misspecified structure, we are able to recover $\rvz$ in a component-wise manner.
\end{examplebox}

\bfsection{Proof Sketch:} The complete proof is deferred to the Appendix~\ref{sec:proof_thm2}. Here we highlight key steps. First, 
the subspace identifiability (Theorem~\ref{thm:submanifold_iden}) imply $\hat{\rvz}=h(\rvz)$, which leads to:
$J_\Phi(\rvz) = J_{\hat{\Phi}}(\hat{\rvz}) J_h(\rvz)$.
Subsequently, by leveraging the injectivity of $\Phi$ in Assumption~\ref{assum:distinct_eigenvalue}, component-wise identifiability is proven by contradiction: any violation would contradict the structural sparsity in assumption \ref{assum:ssparse}.

\begin{examplebox}\bfsection{Assumption Intuitions:}

$\bullet$
Assumption~\ref{assum:independence_z} treats the coordinates of $\rvz$ as distinct latent sources. It provides the additional structure on which the sparsity condition operates.

$\bullet$
Assumption~\ref{assum:ssparse} requires each latent component to be sparsely uniquely localized through the conditional distributions of a subset of observed variables.

$\bullet$
$|\hat{G}_0|\leq |G|$ is the regularization that supports the structural sparsity in Assumption~\ref{assum:ssparse} for learning.
\end{examplebox}

\emph{Assumption~\ref{assum:independence_z} in Proposition~\ref{thm:variabe_iden} claims the independence between each dimension of $\rvz$.
To relax such a constraint, we can alternatively allow latent variables $\rvz$ to exhibit dependence via a known auxiliary domain variable $\rvu$. The details are discussed in Proposition~\ref{thm:varchange_iden} of Sec.~\ref{sec:proof_thm3}.}

%% file: 5_est.tex
\section{Approach}\label{sec:approach}

We now develop an unsupervised estimator for recovering the latent variable $\rvz$ under the data-generating process in Eq.~\ref{eq:dgp}. Our central objective aims to learn $p(\rvx|\rvz)$. However, since $\rvz$ is unobservable from the data, we cannot directly optimize $p(\rvx|\rvz)$ from the data. 
Accordingly, we instead learn 
the joint density in Eq.~\ref{eq:dgp} through formalizing the following:
\begin{align}\label{eq:likelihood}
    p(\rvz,\epsilon,\rvx) & = p_\theta(\rvx|\rvz,\epsilon)p_\gamma(\epsilon|\rvz)p_\delta(\rvz)
\end{align}
where parameters $\theta$ denotes the parameters of $g$. $\gamma$ denotes the parameters of $e$, and $\delta$ parameterizes $p(\rvz)$. To uncover the latent variables $\rvz$ and $\epsilon$ from observed data $\rvx$, we introduce two encoders, $q_{\psi}(\rvz|\rvx)$ and $q_{\phi}(\epsilon|\rvx)$ parameterized by $\psi$ and $\phi$, respectively. 

To the end of learning Eq.~\ref{eq:likelihood}, we build our approach upon the framework of Beta-VAE \citep{betevae_iclr16}.
The overall architecture of our framework is illustrated in Figure~\ref{fig:network}. 
In what follows, we detail each part of our proposed model.

\subsection{Network Design}\label{sec:network}

Learning Eq.~\ref{eq:likelihood} via variational inference suggests the architecture for our approach composing of the following key elements. Specifically, the architecture includes two encoders: $q_\psi(\hat{\rvz}|\rvx)$ for inferring latent variables $\rvz$, and $q_\phi(\hat{\epsilon}|\rvx)$ for estimating the posterior of noise term $\epsilon$. These latent representations are then utilized by a decoder $p_\theta(\hat{\rvx}|\hat{\rvz},\hat{\epsilon})$ to reconstruct the observations $\rvx$. Additionally, we regularize the latent variables by constraining their posterior distributions via the KL divergence to match the learned priors. We detail each of these modules below.

\input{figs/fig_network}

\bfsection{Encoder $q_\psi(\hat{\rvz}|\rvx)$:} 
We parameterize $q_\psi(\hat{\rvz}|\rvx)$ as an isotropic Gaussian characterized by mean $\mu_{\rvz}$ and covariance $\sigma_{\rvz}$. To approximate this posterior, we employ a neural network encoder constructed with an MLP followed by a leaky ReLU activation:
\begin{align}\label{eq:posterior_z}
    \hat{\rvz} \sim \mathcal{N}(\mu_{\rvz}, \sigma_{\rvz}), \quad
    \mu_{\rvz}, \sigma_{\rvz} = \text{LeakyReLU}(\text{MLP}(\rvx))
\end{align}

\bfsection{Encoder $q_\phi(\hat{\epsilon}|\rvx)$:}
Similarly, we parameterize $q_\phi(\hat{\epsilon}|\rvx)$ as another isotropic Gaussian distribution:
\begin{align}\label{eq:posterior_eps}
    \hat{\epsilon} \sim \mathcal{N}(\mu_{\epsilon}, \sigma_{\epsilon}), \quad
    \mu_{\epsilon}, \sigma_{\epsilon} = \text{LeakyReLU}(\text{MLP}(\rvx))
\end{align}

\bfsection{Prior Estimation $p_\delta(\rvz|)$:}
We estimate the prior $p_\delta(\rvz|)$ as a factorized Gaussian across latent dimensions, since we assume the independence of each dimension of $\rvz$ in Assumption~\ref{assum:ind_domains} of Proposition~\ref{thm:variabe_iden}:
\begin{align}
    p_\delta(\rvz)=\prod_{n=1}^N p_\delta(\rvz^n), \quad \rvz^n\sim\mathcal{N}(0, 1)
\end{align}

\bfsection{Prior Estimation $p_\gamma(\epsilon|\rvz)$:}
Direct estimation of the arbitrary density $p_\gamma(\epsilon|\rvz)$ poses a substantial challenge. To overcome this, we introduce a transformation-based module leveraging normalizing flows, representing the prior distribution as a Gaussian transformed via an invertible mapping. 
Suppose each component of $\epsilon$ is independent conditioning on $\rvz$, $\forall m \in [1,M]$, the prior model is formulated through:
$\hat{\eta}^m = \hat{e}^{-1,m}(\hat{\epsilon}^m|\hat{\rvz})$. Using the change-of-variable, the prior distribution of $\hat{\epsilon}^m$ is computed as:
    $p_\gamma(\hat{\epsilon}^m | \hat{\rvz})
    =p(\hat{\eta}^m)\left|\frac{\partial \hat{e}^{-1,m}}{\partial \hat{\epsilon}^m}\right|
    = p_{\gamma}(\hat{e}^{-1,m}(\hat{\epsilon}^m|\hat{\rvz})) \left|\frac{\partial \hat{e}^{-1,m}}{\partial \hat{\epsilon}^m}\right|$.
Aggregating across all dimensions, the complete prior distribution is given by:
\begin{align}\label{eq:prior}
    p_{\gamma}(\hat{\epsilon}| \hat{\rvz}) = \prod_{m=1}^M p(\hat{\eta}^m)\left|\frac{\partial \hat{e}^{-1,m}}{\partial \hat{\epsilon}^m}\right|
\end{align}
The normalizing flow transformation $\hat{e}$ is implemented using a stacked MLP.

\bfsection{Decoder $p_\theta(\hat{\rvx}|\hat{\rvz},\hat{\epsilon})$:}
The decoder generates the reconstructed observations $\hat{\rvx}$ from inferred latent variables $\hat{\rvz}$ and $\hat{\epsilon}$. It is implemented using an MLP followed by leaky ReLU activations:
\begin{align}\label{eq:decoder}
    \hat{\rvx} = \text{LeakyReLU}(\text{MLP}(\hat{\rvz},\hat{\epsilon}))
\end{align}

\subsection{Training Objective}

Our complete ELBO is:
\begin{align}\label{eq:elbo}
    \mathcal{L}_{\text{ELBO}} = \mathbb{E}_{\hat{\rvz}\sim q_{\psi}}[\mathcal{L}_{\text{cond}}] -\beta_1\underbrace{\mathbb{E}_{\hat{\rvz}\sim q_{\psi}} \big(\log q\left(\hat{\rvz}\vert \rvx \right) - \log p_\delta(\rvz)\big)}_{\mathcal{L}_{\text{KLD}}\text{\HS of $\rvz$}} + \mathcal L_{\text{sparse}}
\end{align}
where $\beta_1$ denotes the hyperparameter to penalize the KL divergence of $\rvz$. 
In the following, we explain each part in details.

We first construct the variational lower bound of $p(\rvx|\rvz)$. 
For a fixed $\hat{\rvz}$, we can define:
\begin{align}\label{elbo_cond}
    \mathcal{L}_{\text{cond}}=\mathbb{E}_{\hat{\epsilon} \sim q_{\phi}}\left[\log p_{\theta}(\hat{\rvx}|\hat{\rvz},\hat{\epsilon})\right]-\beta_2\mathbb{E}_{\hat{\epsilon} \sim q_{\phi}} \big(\log q\left(\hat{\epsilon}\vert \rvx \right) - \log p_\gamma(\hat{\epsilon}|\hat{\rvz})\big)
\end{align}
where $\beta_2$ is the hyperparameter that balance the KL divergence of $\epsilon$. 

Additionally, to learn the latent representation $\hat{\rvz}$, we further regularize its approximate posterior $q(\hat{\rvz}|\rvx)$ toward the prior $p_\delta(\rvz)$ using the KL divergence. We then take the expectation of both $\mathcal{L}_{\mathrm{cond}}$ and ${\text{KLD}}\text{\HS of $\rvz$}$ with respect to $q_{\psi}(\hat{\rvz}\mid\rvx)$.

We also design the regularizer $\mathcal L_{\text{sparse}}$ by:
\begin{align}
    \mathcal L_{\text{sparse}}=|\frac{\mathbb{E}_{\hat{\epsilon} \sim q_{\phi}}\left[\partial\log p_{\theta}(\hat{\rvx}|\hat{\rvz},\hat{\epsilon})\right]}{\partial \hat{\rvz}} + \frac{\mathbb{E}_{\hat{\epsilon} \sim q_{\phi}}\left[\partial\log p_{\gamma}(\hat{\epsilon}|\hat{\rvz})\right]}{\partial \hat{\rvz}}|_1 
\end{align}
$|*|_1$ denotes the $\ell_1$ norm that enforces the sparsity regularization in the Regularization~\ref{assum:sparsity}. 
We use the $\log p_{\theta}(\hat{\rvx}|\hat{\rvz},\hat{\epsilon})$ and $\log p_{\gamma}(\hat{\epsilon}|\hat{\rvz})$ instead of the actual $p_{\theta}(\hat{\rvx}|\hat{\rvz})$ is given the fact that
$\frac{\partial p(\rvx^i\mid\rvz)}{\partial\rvz^j}=0
\Longleftrightarrow
\frac{\partial\log p(\rvx^i\mid\rvz)}{\partial\rvz^j}=0$.
where the $\ell_1$ penalty serves as a differentiable surrogate for the $\ell_0$ norm. 
{\it Notably, we do not require any knowledge of the theoretical partitions of $\rvx$ in Definition~\ref{def:misspeci}}.

%% file: figs/fig_network.tex
\begin{wrapfigure}{r}{5cm}
\includegraphics[width=\linewidth]{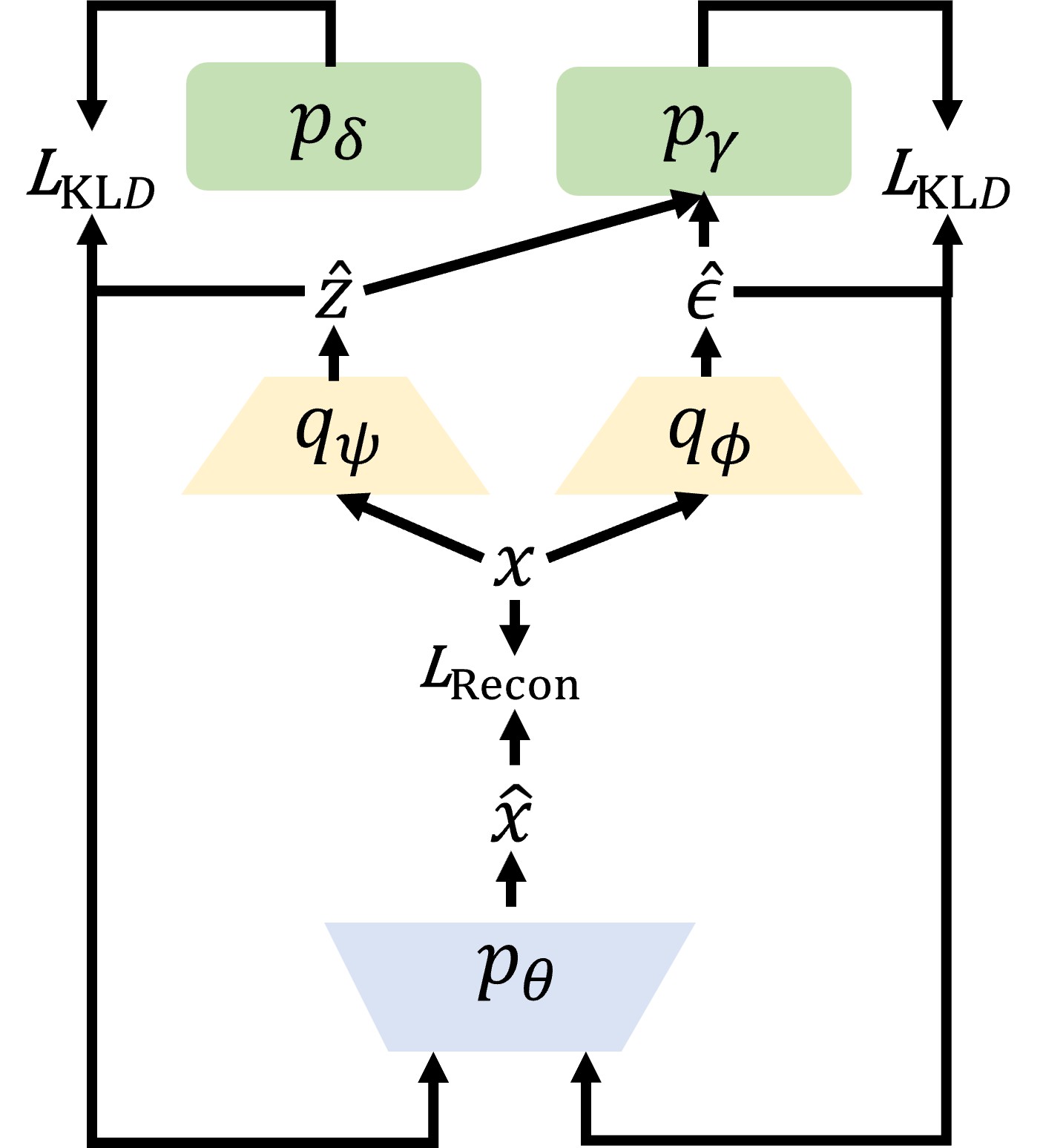}
\caption{The overall framework of our proposed approach 
consists of: (1) two encoders $q_\psi$ and $q_\phi$ that map observations $\rvx_t$ to $\hat{\rvz}$ and $\hat{\epsilon}$, respectively; (2) a decoder that reconstructs observations $\hat{\rvx}$ from $\hat{\rvz}$ and $\hat{\epsilon}$; and (3) two prior estimation modules $p_\delta$ and $p_\gamma$ that models the prior of $\rvz$ and $\epsilon$, respectively.
We train the framework by $L_\text{Recon}$ along with $L_\text{KLD}$.
}
\label{fig:network}
\vspace{-10px}
\end{wrapfigure}

%% file: 6_exp.tex
\section{Experiments}\label{sec:exp}

\subsection{Experimental Setup}

\bfsection{Data:}
We evaluate the precise and approximate identifiability results using synthetic data described in Section~\ref{sec:synthetic}, and the assumption justifications are provided in Section~\ref{sec:assumjust}. 
For precise identifiability, since Proposition~\ref{thm:variabe_iden} is established upon the subspace result of Theorem~\ref{thm:submanifold_iden}, we create a dataset evaluates component-wise recovery of $\rvz$.
For approximate identifiability, we synthesize {\bf{\emph{another}}} dataset satisfying the assumptions required by Theorem~\ref{thm:approx_vareps}. This setting evaluates approximate subspace recovery only, since approximate component-wise identifiability is left for future work.

\bfsection{Metrics:}
For precise identifiability, we use the Mean Correlation Coefficient (MCC) to evaluate component-wise recovery. MCC measures the average absolute correlation between matched coordinates of the true and estimated latent variables. MCC scores range from 0 to 1, with larger values indicating better component-wise identification.

For approximate identifiability, we report the coefficient of determination $R^2$ after aligning $\hat{\rvz}$ with $\rvz$. 
A higher value of $R^2$ corresponds to a smaller discrepancy between $\hat{\rvz}$ and its admissible equivalent $\rvz$ and therefore indicates better approximate subspace identifiability.

\bfsection{Comparison Methods.}
We compare our method with several representative approaches for latent
variable identification. IndVAE~\citep{hu2008instrumental,zheng2025nonparametric}
assumes conditional independence among observations given the latent
variables; its objective is detailed in
Section~\ref{sec:obj_indvae}. MCRL~\citep{sun2025causal} assumes that
the noise term $\epsilon$ is independent of $\rvz$. We additionally
compare against Beta-VAE~\citep{betevae_iclr16},
iVAE~\citep{ivae}, and SLOW-VAE~\citep{slowvae}.
We further conduct ablations to asses the impact of latent-dependent noise and structural sparsity. ``W/O $e$'' removes the dependence $\epsilon=e(\rvz)$ from Eq.~\eqref{eq:dgp}, whereas ``W/O s'' removes the structural-sparsity regularization imposed on $\Phi$. Section~\ref{sec:synexp_imt} provides the implementation details of our approach.

\subsection{Results and Discussions:}

\input{figs/fig_combine}

\bfsection{Precise Identifiability:}
As shown in Figure~\ref{fig:vis}(b), our method achieves the highest MCC among all compared approaches, demonstrating more accurate component-wise recovery of the latent variables under misspecified structure. Figure~\ref{fig:vis}(a) provides a complementary visualization: each true latent coordinate $\rvz^i$ exhibits a clear one-to-one correspondence with a single estimated coordinate $\hat{\rvz}^i$, while the off-diagonal pairs show substantially weaker dependence. This behavior is consistent with component-wise recovery.

The ablation results further illustrate the importance of modeling the data-generating structure. IndVAE performs substantially worse because its conditional-independence assumption is violated by our synthetic data. Removing the latent-dependent noise model (``W/O $e$'') also reduces MCC comparing with our full model, indicating that explicitly modeling $\epsilon=e(\rvz,\eta)$ is important for separating the latent signal from the dependent noise. Also, ``W/O s'' yields lower MCC than the full model, supporting the impact of the structural-sparsity regularization.

\input{figs/fig_z}

\bfsection{Approximate Identifiability:}
We use the same model architecture and the same training objective as in the precise-identifiability experiment; only $\overline{Per}<\eta$ does not hold in the the synthetic data. This setting therefore evaluates whether the same estimator can still recover a representation close to the admissible equivalence class under the approximate identifiability regime.

Figure~\ref{fig:RR} reports the $R^2$ scores. Our method achieves the highest $R^2$ among all compared approaches, indicating that $\hat{\rvz}$ retains the strongest correspondence with $\rvz$ in the approximate-identifiability regime.
This is consistent with Theorem~\ref{thm:approx_vareps}.

The comparison with IndVAE, MCRL, Beta-VAE, iVAE and SLOW-VAEfurther shows that explicitly modeling the misspecified
latent-dependent-noise process improves latent recovery even outside the precise-identifiability regime. In particular, ``W/O $e$'' obtains a lower $R^2$ than the complete model, showing that modeling the dependence between $\epsilon$ and $\rvz$ remains beneficial for approximate subspace recovery. ``W/O s'' also decreases $R^2$, suggesting the benefit from structural sparsity.

{\bf\emph{Also, we postpone discussing the real-world experiments, and more ablation studies in Section~\ref{sec:exp_u}}}.

%% file: figs/fig_combine.tex
\begin{figure*}
\centering
\includegraphics[width=\linewidth]{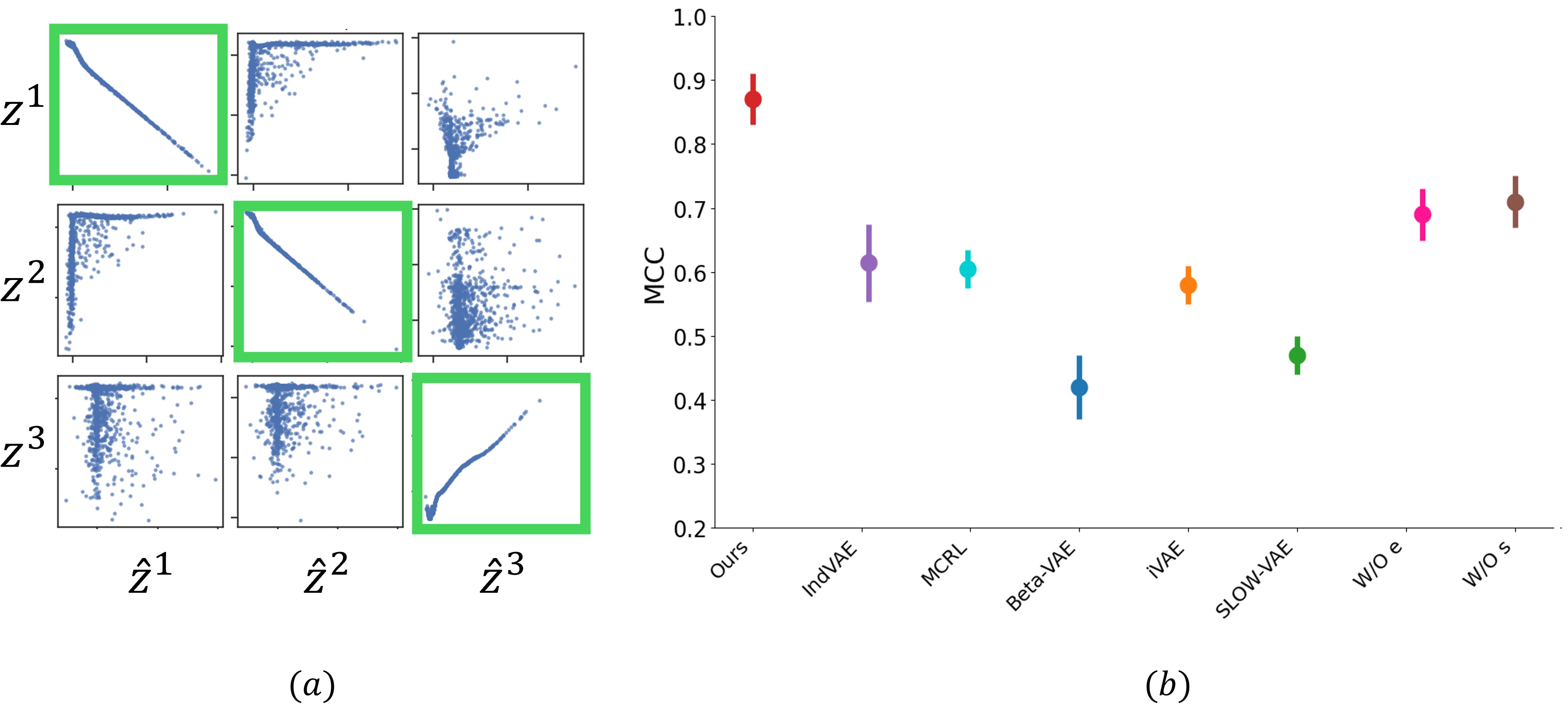}
\caption{
(a) Visualization of the correlations between each component of true latent variables ($\rvz^i$) and their corresponding component of estimated latent variables ($\hat{\rvz}^i$) using our approach. 
The green bounding boxes highlight the components that are identified. (b) Mean Correlation Coefficient (MCC) scores comparing our framework with state-of-the-art approaches, including IndVAE, MCRL, BetaVAE, iVAE, and SlowVAE, as well as the ablation baselines W/O $e$ and W/O $s$. 
}
\label{fig:vis}
\vspace{-20px}
\end{figure*}

%% file: figs/fig_z.tex
\begin{wrapfigure}{r}{8.0cm}
\vspace{-8px}
\includegraphics[width=8.0cm]{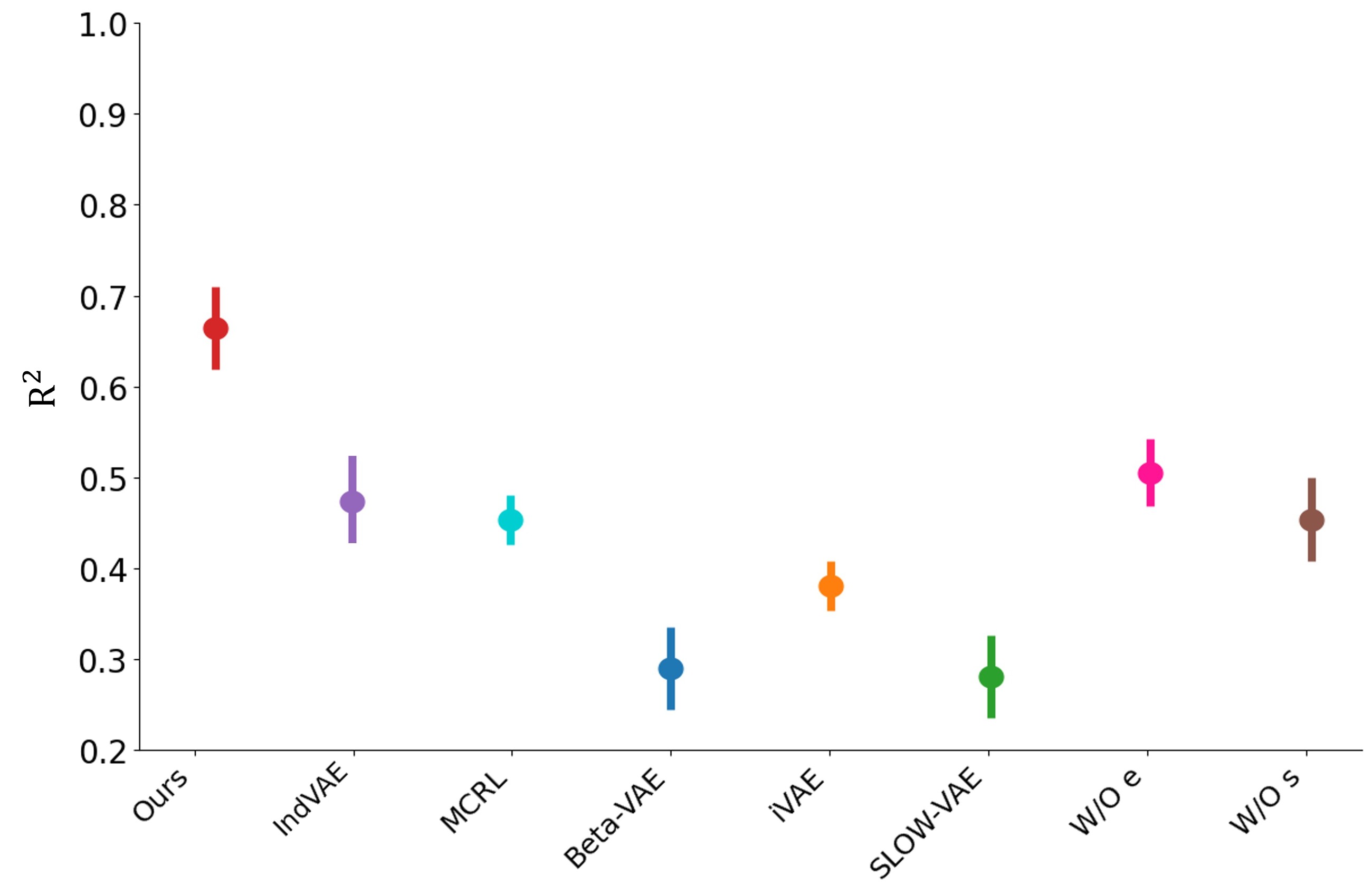}
\caption{
$R^2$ scores for comparison methods.
}
\label{fig:RR}
\vspace{-7px}
\end{wrapfigure}

%% file: 7_appendix.tex
\newpage
\appendix

\section{Notions}

\begin{table}[h]
\centering
\small
\def\arraystretch{1.3}
\begin{tabular}{l c c c}
\multicolumn{4}{c}{\textit{Table of notions}} \\
\toprule
\multicolumn{4}{c}{\textit{Variables}} \\
\toprule
$\displaystyle \rvx\in\mathbb{R}^K$ & Observations &
$\displaystyle \hat{\rvx}\in\mathbb{R}^K$ & Reconstructions\\
$\displaystyle \kappa$ & $||L_{\rvx_\rva|\rvz}||_{op}||L^{-1}_{\rvx_\rva|\rvz}||_{op}$ &
$\displaystyle \varepsilon$ & Identifiability error
\\
$\displaystyle \rvz\in\mathbb{R}^N$ & Latent variables &
$\displaystyle \hat{\rvz}\in\mathbb{R}^N$ & Latent variable estimations\\
$\displaystyle \epsilon$ & True dependent noise term &
$\displaystyle \hat{\epsilon}$ & Estimation of $\epsilon$\\
$\displaystyle \eta$ & Auxiliary variable &
$\displaystyle \hat{\eta}$ & Estimation of $\eta$\\
\midrule
\multicolumn{4}{c}{\textit{Indices}} \\
\midrule
$\displaystyle \{\rva,\rvb,\rvc\}$ & The indices of partitions of $\rvx$ & 
$\displaystyle \{A, B, C\}$ & The indices of partitions of $\rvx'$
\\
$\displaystyle n\in [N]$ & The indices of $\rvz$ & 
$\displaystyle \hat{n}\in [N]$ & The indices of $\hat{\rvz}$\\
$\displaystyle i$ & The index of $\rho_\Lambda$ and $\rho$ & 
$\displaystyle j$ & The index of $\rho_\Lambda$ and $\rho$ \\
\midrule
\multicolumn{4}{c}{\textit{Operators}} \\
\midrule
$\displaystyle L$ & Integral linear operator &
$\displaystyle \Lambda$ & Multiplication operator \\
$\displaystyle {\rho_\Lambda}$ & The eigenvalue of$L_{\rvx_\rva|\rvz}L_{\rvx_\rvb|\rvz}L^{-1}_{\rvx_\rva|\rvz}$ &
$\displaystyle \rho$ & The eigenvalue of $L_{\rvx_\rva,\rvx_\rvb|\rvx_\rvc}L^{-1}_{\rvx_\rva|\rvx_\rvc}$ \\
$Per$ & Perturbation operator &
$\overline{Per}$ & The upper bound of $||Per||_{op}$ \\
\midrule
\multicolumn{4}{c}{\textit{True \& learned model}} \\
\midrule
$\displaystyle g$ & True mixing function &
$\displaystyle \hat{g}$ & Learned mixing function \\
$\displaystyle e $ & True function of $\epsilon$ &
$\displaystyle \hat{e}$ & Learned function of $\epsilon$\\
$J_\Phi$ & The jacobian of $_\Phi$ &
$J_{\hat{_\Phi}}$ & The jacobian of $\hat{g_\Phi}$ \\
\midrule
\multicolumn{4}{c}{\textit{Optimizations}} \\
\midrule
$\displaystyle \psi$ & Parameters of posterior $q_\psi(\rvz|\rvx)$ &
$\displaystyle \phi$ & Parameters of posterior $q_\phi(\epsilon|\rvx)$ \\
$\displaystyle \delta$ & Parameters of prior $p_\delta(\rvz)$ &
$\displaystyle \gamma$ & Parameters of prior $p_\gamma(\epsilon|\rvz)$ \\
$\displaystyle \theta$ & Parameters of decoder $p_\theta(\rvx|\rvz)$ &
$\displaystyle \vert * \vert_{1}$ & $l_1$ norm of $*$ \\
\bottomrule
\end{tabular}
\end{table}

\section{Proof of Theorems}

\subsection{Previous Results}
\label{sec:preliminaries}

In this section, we recapitulate and summarize the previous results from \citep{hu2008instrumental,zheng2025nonparametric}. 

{\it

For $\rvx'=\{\rvx'_A, \rvx'_B, \rvx'_C\}$, where $\{\rvx'_A, \rvx'_B, \rvx'_C\}$ denote the three disjoint partitions, $ p(\rvx'|\rvz') = p(\rvx'_A|\rvz')p(\rvx'_B|\rvz')p(\rvx'_C|\rvz')$ \footnote{\bf \emph We emphasize that the partition in both previous and our work is an existential structural condition for the identifiability analysis. We do not require the partition labels, nor separate access to $\rvx_\rva$, $\rvx_\rvb$, and $\rvx_\rvc$, during learning and inference.};

Consider observed variables $\rvx' \in \mathbb{R}^K$ and the estimated latent variables $\hat{\rvz}'\in\mathbb{R}^N$, suppose that there exist functions $\hat{g}'$ and $\hat{e}'$, 
and the following assumptions hold:
\begin{enumerate}
    \item The joint distribution of $(\rvx, \rvz)$ admits a bounded density with respect to a suitable product measure defined on their supports. Furthermore, all marginal and conditional densities derived from this joint distribution are also bounded;
    \item The operators $L_{\rvx'_A|\rvz'}$ and $L_{\rvx'_A|\rvx'_B}$ are injective;
    \item $\forall \underline{\rvz}'\neq \bar{\rvz}', p(\rvx'_C;\underline{\rvz}')\neq p(\rvx'_C;\bar{\rvz}')$;
\end{enumerate}
then $\rvz'$ must be identified up to an invertible transformation $h'$.
}

\bfsection{Proof:}
Given Assumption {\it{1}}, we can obtain the following:  \begin{align} 
p_{\rvx'_C \rvx'_A | \rvx'_B}(\rvx'_C, \rvx'_A | \rvx'_B) & = \int p_{\rvx'_C \rvx'_A \rvz' | \rvx'_B}(\rvx'_C, \rvx'_A, \rvz' | \rvx'_B) d\rvz \nonumber\\ & = \int p_{\rvx'_C | \rvx'_A \rvz' \rvx'_B}(\rvx'_C | \rvx'_A, \rvz', \rvx'_B) p_{\rvx'_A \rvz' | \rvx'_B}(\rvx'_A, \rvz' | \rvx'_B) d\rvz \nonumber\\ & = \int p_{\rvx'_C | \rvx'_A \rvz'}(\rvx'_C | \rvx'_A, \rvz') p_{\rvx'_A \rvz' | \rvx'_B}(\rvx'_A, \rvz' | \rvx'_B) d\rvz \nonumber\\ & = \int p_{\rvx'_C | \rvx'_A \rvz'}(\rvx'_C | \rvx'_A, \rvz') p_{\rvx'_A | \rvz' \rvx'_B}(\rvx'_A | \rvz', \rvx'_B) p_{\rvz' | \rvx'_B}(\rvz' | \rvx'_B) d\rvz \nonumber\\ & = \int p_{\rvx'_C | \rvx'_A \rvz'}(\rvx'_C | \rvx'_A, \rvz') p_{\rvx'_A | \rvz'}(\rvx'_A | \rvz') p_{\rvz' | \rvx'_B}(\rvz' | \rvx'_B) d\rvz \nonumber\\ & = \int p_{\rvx'_C | \rvz'}(\rvx'_C | \rvz') p_{\rvx'_A | \rvz'}(\rvx'_A | \rvz') p_{\rvz' | \rvx'_B}(\rvz' | \rvx'_B) d\rvz 
\end{align} 

Leveraging this eqation and Definition~\ref{def:integral_operator}, we can derive the following operator: 
\begin{align} 
(L_{\rvx'_C;\rvx'_A|\rvx'_B} f')(\rvx'_A) &= \int \int p_{\rvx'_A | \rvz'}(\rvx'_A | \rvz') p_{\rvx'_C | \rvz'}(\rvx'_C | \rvz') p_{\rvz' | \rvx'_B}(\rvz' | \rvx'_B) f'(\rvx'_B) d\rvx'_B d\rvz \nonumber\\ & = \int p_{\rvx'_A | \rvz'}(\rvx'_A | \rvz') (\Lambda_{\rvx'_C;\rvz'} L_{\rvz'|\rvx'_B} f')(\rvz') d\rvz \nonumber\\ & = (L_{\rvx'_A|\rvz'} \Lambda_{\rvx'_C;Z} L_{\rvz'|\rvx'_B} f')(\rvx'_A) 
\end{align} 

The above equation indicates: 
\begin{align} \label{eq:oper_eq}
L_{\rvx'_C;\rvx'_A|\rvx'_B} = L_{\rvx'_A|\rvz'}\Lambda_{\rvx'_C;\rvz'} L_{\rvz'|\rvx'_B} 
\end{align} 
This equivalence holds over some functions space $ \mathcal{G}(\mathcal{Z}) $, given the factorization properties of the conditional densities established earlier.

Now, integrating over $ \rvx'_C $, by using the fact that: $\int L_{\rvx'_C;\rvx'_A|\rvx'_B} f'(\rvx'_C) d\rvx'_C = L_{\rvx'_A|\rvx'_B} f'$ we can obtain: 
\begin{align} 
L_{\rvx'_A|\rvx'_B} f'(\rvx'_A) & = \int p_{\rvx'_A | \rvx'_B}(\rvx'_A | \rvx'_B) f'(\rvx'_B) d\rvx'_B \nonumber\\ & = \int\int p_{\rvx'_A | \rvz', \rvx'_B}(\rvx'_A |\rvz', \rvx'_B)p_{\rvz'|\rvx'_B}(\rvz'|\rvx'_B) f'(\rvx'_B) d\rvx'_B d\rvz \nonumber\\ & = \int p_{\rvx'_A | \rvz'}(\rvx'_A |\rvz')[L_{\rvz'|\rvx'_B} f'] d\rvz \nonumber\\ & = L_{\rvx'_A|\rvz'}L_{\rvz'|\rvx'_B} f'](\rvx'_A)
\end{align} 
where the second equation leverages $ \rvx'_A \perp \rvx'_B \mid \rvz'$. By assuming the inejctivity of $ L_{\rvx'_A|\rvz'} $ in Assumption {\it{2}}, we can arrive at $L_{Z|\rvx'_B} = L_{\rvx'_A|\rvz'}^{-1} L_{\rvx'_A|\rvx'_B}$. Substitute this to Eq. 15: 
\begin{align} 
L_{\rvx'_C;\rvx'_A|\rvx'_B}L_{\rvx'_A|\rvx'_B}^{-1} = L_{\rvx'_A|\rvz'} \Lambda_{\rvx'_C;Z} L_{\rvx'_A|\rvz'}^{-1}
\end{align} 
where the LHS involves only observable variables, and the RHS explicitly depends on the latent variable $\rvz'$.
$\Lambda_{\rvx'_C;\rvz'}$ determines the eigenvalues of $L_{\rvx'_A | Z}\Lambda_{\rvx'_C;\rvz'}L^{-1}_{\rvx'_A | \rvz'}$, whose diagonal entries correspond to the conditional distributions $p(\rvx'_C | \rvz')$. Each $\rvz'$ indexing a distinct conditional distribution of $p(\rvx'_C | \rvz')$.
Under Assumption {\it{3}}, where $p(\rvx'_C \mid \rvz')$ are distinct for different values of $\rvz'$, the eigenvalues are distinct. This allows a bijective mapping $h': \mathcal{Z} \to \mathcal{Z}$ to permute $\rvz'$ while preserving the values of $p(\rvx'_C \mid \rvz')$. Therefore, the latent variable can only be recovered up to such a permutation, i.e., $\hat{\rvz'} = h'(\rvz')$, which yields the identifiability up to an invertible transformation $h'$.

\subsection{Proof of Proposition 1 and Theorem 1}\label{proof:prop1}

\bfsection{Proposition 1:}
Suppose Assumptions~\ref{assum:error} $\sim$ \ref{assum:distinct_eigenvalue} hold, and if both $\varrho^i$ and $\varrho^i_\Lambda$ include only $\rho^i$ and $\rho_\Lambda^i$, respectively. Then, for all $i>1$ and for all $\rvz$,
\begin{align}\label{eq:rho_1}
    \rho^i_\Lambda\in \left(\max\{\rho^i-\kappa\overline{\it{Per}}, \rho^{i-1}+\kappa\overline{\it{Per}}\}, \rho^i+\kappa\overline{\it{Per}}\right)
\end{align}
and if $i=1$, for all $\rvz$,
\begin{align}\label{eq:rho_2}
    \rho^i_\Lambda\in (\rho^i-\kappa\overline{\it{Per}}, \rho^i+\kappa\overline{\it{Per}})
\end{align}
guarantee the distinct spectral values of $L_{\rvx_\rva|\rvz}\Lambda_{\rvx_\rvb|\rvz}L^{-1}_{\rvx_\rva|\rvz}$.

\bfsection{Proof:}


Recall that we define:
\begin{align}
    \eta = \min_{i\neq j}\frac{|\rho^i - \rho^j|}{2\kappa} - \alpha > 0 \quad \text{and} \quad \overline{Per} < \eta,
\end{align}
where $\alpha>0$ denotes constant, and $\kappa=||L_{\rvx_\rva|\rvz}||_{op}||L^{-1}_{\rvx_\rva|\rvz}||_{op}$.
Notably, $\kappa\geq 1$,
since $||L_{\rvx_\rva|\rvz}L^{-1}_{\rvx_\rva|\rvz}||_{op}=||I||_{op}\leq ||L_{\rvx_\rva|\rvz}||_{op}||L^{-1}_{\rvx_\rva|\rvz}||_{op}$, where $I$ denotes the identity operator.

Suppose that two distinct spectral values, denoted $\rho^1_\Lambda$ and $\rho^2_\Lambda$, fall within the same interval $(\rho^i - \kappa\overline{Per}, \rho^i + \kappa\overline{Per})$. we have:
\begin{align}
    |\rho^1_\Lambda - \rho^i| &\leq \kappa\overline{Per} < |\rho^i - \rho^j|, \nonumber\\
    |\rho^2_\Lambda - \rho^i| &\leq \kappa\overline{Per} < |\rho^i - \rho^j|.
\end{align}

Applying the triangle inequality, we obtain:
\begin{align}\label{eq:rho_3}
    |\rho^1_\Lambda - \rho^2_\Lambda| \leq |\rho^1_\Lambda - \rho^i| + |\rho^2_\Lambda - \rho^i| < |\rho^i - \rho^j| = 2\kappa\eta.
\end{align}

However, inequality \ref{eq:rho_3} implies:
\begin{align}
    \frac{|\rho^1 - \rho^2|}{2\kappa} < \eta,
\end{align}
which directly contradicts the definition of $\eta$. Thus, the uniqueness of each spectral value $\rho^i_\Lambda$ within its respective interval is proved.


We reorganize the Theorem~1 below:

{\it\bfsection{Theorem 1}
Consider observed variables $\rvx \in \mathbb{R}^K$ and the estimated latent variables $\hat{\rvz}\in\mathbb{R}^N$, suppose 
the following assumptions hold:
\begin{enumerate}[i]
    \item\label{assum:bound} The joint distribution of $(\rvx, \rvz)$ admits a bounded density with respect to a suitable product measure defined on their supports. Furthermore, all marginal and conditional densities derived from this joint distribution are also bounded;
    \item\label{assum:subiden_ijc_supp} The operators $L_{\rvx_\rva|\rvz}$ and $L_{\rvx_\rva|\rvx_\rvc}$ are injective and bounded-below;
    \item\label{assum:subiden_car_supp} $L_{\rvx_\rva,\rvx_\rvb|\rvx_\rvc}L^{-1}_{\rvx_\rva|\rvx_\rvc}$ has distinct spectral values with cardinality equal to that of $\Lambda_{\rvx_\rvb|\rvz}$.  Also,  Let $\Theta:\mathcal{Z}\rightarrow \mathcal{D}, s.t. p(\rvx_\rvb|\rvz)=\Theta(\rvz)$ denotes the map from $\mathcal{Z}$ to $\mathcal{D}$. $\Theta$ is injective;
    \item\label{assum:M_appendix} There exists an operator $M$ such that $M(L_{\rvx_\rvb|\rvz})=M(L_{\rvx_\rvb|\tilde{h}(\rvz)})=t(\rvz)$, where $t$ is a differentiable transformation.
\end{enumerate}
then for $\tilde{h}\in\tilde{\mathcal{H}}$ and $t\in\mathcal{T}$ $(\tilde{\mathcal{H}}\cap\mathcal{T}\neq\emptyset)$, 
if $h \in \tilde{\mathcal{H}}\cap\mathcal{T}\Rightarrow \hat{\rvz}=h(\rvz)=\tilde{h}(\rvz)=t(\rvz)$.
In other words, $\rvz$ must be subspace identified.
}

\bfsection{Proof:}\label{proof:thm4}

Our goal is to demonstrate subspace identifiability of $\rvz$ from the data generated process in Eq.~\ref{eq:dgp}. To such an end, we proceed in the following steps:

\textit{Step 1: Operator Construction and Spectral Decomposition.}

Our approach follows the previous results in Sec.~\ref{sec:preliminaries}, and proceeds to decompose the bounded linear operator $L_{\rvx_\rva,\rvx_\rvb|\rvx_\rvc}$.  
The structural misspecification in Definition~\ref{def:misspeci} violates the conditional independence by $p(\rvx|\rvz)\neq p(\rvx_\rva|\rvz)p(\rvx_\rvb|\rvz)p(\rvx_\rvc|\rvz)$, thus introduces a discrepancy between $L_{\rvx_\rva,\rvx_\rvb|\rvx_\rvc}$ and $L_{\rvx_\rva|\rvz}L_{\rvx_\rvb|\rvz}L_{\rvz|\rvx_\rvc}$ according to Eq.~\ref{eq:oper_eq}. 
Consequently, we can define the difference operators:
\begin{align}\label{eq:error_oper_1}
    & \rvd(\rvx_\rvb) = L_{\rvx_\rva,\rvx_\rvb|\rvx_\rvc} - L_{\rvx_\rva|\rvz}L_{\rvx_\rvb|\rvz}L_{\rvz|\rvx_\rvc}, \nonumber \\
    & \rvd = \int \rvd(\rvx_\rvb) d\rvx_\rvb = L_{\rvx_\rva|\rvx_\rvc} - L_{\rvx_\rva|\rvz}L_{\rvz|\rvx_\rvc}.
\end{align}

Assumption~\ref{assum:bound} guarantees boundedness of each term in Eq.~\ref{eq:error_oper_1}. Leveraging these, we rewrite:
\begin{align}\label{eq:error_derivation_appendix}
    L_{\rvx_\rva|\rvz}\Lambda_{\rvx_\rvb|\rvz}L^{-1}_{\rvx_\rva|\rvz}
    &= (L_{\rvx_\rva,\rvx_\rvb|\rvx_\rvc} - \rvd(\rvx_\rvb))(L_{\rvx_\rva|\rvx_\rvc} - \rvd)^{-1} \nonumber\\
    &= L_{\rvx_\rva,\rvx_\rvb|\rvx_\rvc}L^{-1}_{\rvx_\rva|\rvx_\rvc} + \it Per,
\end{align}
where $\it Per$ represents the perturbation term arising from the violation of conditional independence.

\textit{Step 2: Spectral resolution of $L_{\rvx_\rva|\rvz}\Lambda_{\rvx_\rvb|\rvz}L^{-1}_{\rvx_\rva|\rvz}$}

At this point, we aim to understand if $L_{\rvx_\rva|\rvz}\Lambda_{\rvx_\rvb|\rvz}L^{-1}_{\rvx_\rva|\rvz}$ is unique spectral decomposition of RHS of Eq.~\ref{eq:error_derivation_appendix} with $Per\neq 0$.
We take inspiration from \citep{dunford1971linear} to proceed with the necessary and sufficient condition by using the projection-valued measure in the following steps. 

Assumption~\ref{assum:bound} along with the fact that $p(\rvx_\rvb|\rvz)$ being real-valued 
suggest that the multiplication operator $\Lambda_{\rvx_\rvb|\rvz}$ is bounded self-adjoint, and therefore admits a unique orthogonal projection-valued measure $E_{\rvx_\rvb}$ such that 
\begin{align} \Lambda_{\rvx_\rvb|\rvz}=\int_{\sigma(\Lambda_{\rvx_\rvb|\rvz})}\lambda\,dE_{\rvx_\rvb}(\lambda)
\end{align}
where $\sigma$ denotes the spectral value, $(E_{\rvx_\rvb}(\Delta)f)(\rvz)=\mathbf 1_{\{p(\rvx_\rvb\mid \rvz)\in\Delta\}}f(\rvz)$ for every Borel set $\Delta\subseteq\mathbb R$, $1$ denotes the indicator and $f$ denotes some function. 

For $L_{\rvx_\rva|\rvz}\Lambda_{\rvx_\rvb|\rvz}L^{-1}_{\rvx_\rva|\rvz}$, define the projection-valued measure $P_{\rvx_\rvb}$ by:
\begin{align}
    P_{\rvx_\rvb}(\Delta)\coloneqq L_{\rvx_\rva|\rvz} E_{\rvx_\rvb}(\Delta)L_{\rvx_\rva|\rvz}^{-1}
\end{align}
Given the properties of projection-valued measure, we can obtain: 
$P_{x_b}(\mathbb R)=I_{\mathcal K}$ ($\mathcal K$ denotes the range of $L_{\rvx_\rva|\rvz}$). These equations yield 
\begin{align}\label{eq:def_pvm}
    L_{\rvx_\rva|\rvz}\Lambda_{\rvx_\rvb|\rvz}L^{-1}_{\rvx_\rva|\rvz}=\int_{\sigma(L_{\rvx_\rva|\rvz}\Lambda_{\rvx_\rvb|\rvz}L^{-1}_{\rvx_\rva|\rvz})}\lambda\,dP_{\rvx_\rvb}(\lambda)
\end{align} 
with $\sigma(L_{\rvx_\rva|\rvz}\Lambda_{\rvx_\rvb|\rvz}L^{-1}_{\rvx_\rva|\rvz})=\sigma(\Lambda_{\rvx_\rvb|\rvz})$. Moreover, $\operatorname{Ran}P_{\rvx_\rvb}(\Delta)=L_{\rvx_\rva|\rvz}\operatorname{Ran}E_{x_b}(\Delta)$ and $\ker P_{x_b}(\Delta)=L_{\rvx_\rva|\rvz}\operatorname{Ran}E_{x_b}(\mathbb R\setminus\Delta)$, where $Ran$ denote the range. Hence, $P_{\rvx_\rvb}(\Delta)$ is uniquely determined by its range and null space. 



Since $E_{x_b}(\Delta)$ is a spectral projection, the Hilber space $\mathcal H_z$ of $\Lambda_{\rvx_b|\rvz}$ is
$\mathcal H_z=
\operatorname{Ran}E_{\rvx_\rvb}(\Delta)\oplus
\operatorname{Ran}E_{\rvx_\rvb}(\mathbb R\setminus\Delta)$.
By invertibility of $L_{\rvx_a\mid \rvz}$,
$$
\mathcal K=
\operatorname{Ran}P_{\rvx_\rvb}(\Delta)
\oplus
\ker P_{\rvx_\rvb}(\Delta).
$$
Suppose
$w\in\operatorname{Ran}P_{x_b}(\Delta)\cap\ker P_{x_b}(\Delta)$.
Since $w\in\operatorname{Ran}P_{x_b}(\Delta)$, there exists some $a\in\mathcal K$ such that
$w=P_{x_b}(\Delta)a$. Hence,
$
P_{x_b}(\Delta)w
=
P_{x_b}(\Delta)^2a
=
P_{x_b}(\Delta)a
=
w.
$
On the other hand, since $w\in\ker P_{x_b}(\Delta)$,
$P_{x_b}(\Delta)w=0$. Therefore $w=0$, implying
$$
\operatorname{Ran}P_{x_b}(\Delta)\cap
\ker P_{x_b}(\Delta)=\{0\}.
$$
Hence, every $f\in\mathcal K$ admits a unique decomposition
$f=u+v$, where
$u\in\operatorname{Ran}P_{\rvx_\rvb}(\Delta)$ and
$v\in\ker P_{\rvx_\rvb}(\Delta)$, 
$\Rightarrow P_{\rvx_\rvb}(\Delta)f=u$.
Therefore, any projection with the same range and null space must map $f=u+v$ to its unique range component $u$, and is therefore identical to $P_{\rvx_\rvb}(\Delta)$. 
Thus, for a fixed data generating process in Eq.~\ref{eq:dgp}, Eq.~\ref{eq:error_derivation_appendix} admits an unique spectral projections. 

\textit{Step 3: Connecting Bijection $\tilde{h}$ with Differentiable Transformation $h$.}

Proposition~\ref{prop:eigenvalue} has proven the distinctness of $\rho_\Lambda$, which is determined by $p(\rvx_\rvb|\rvz)$. In combination of the injectivity map assumption in Assumption~\ref{assum:subiden_car_supp}, $p(\rvx_\rvb|\rvz)$ corresponds to a particualr $\rvz$. 
Since $\rho_\Lambda$ is the essential range of $p(\rvx_\rvb|\rvz)$, it remains invariant by re-labeling through permuting $\rvz$. This can be formally described by $\rvz'=h'(\rvz)$, where $h'$ denotes a bijection $h':\mathcal{Z}\rightarrow\mathcal{Z}$, and $\rvz'$ is a permuted version of $\rvz$.

Definition~\ref{def:subspace_iden} requires the invertible transformation $h$ to be differentiable. However, $\tilde{h}$ alone may not satisfy this differentiability constraint. To resolve this, we invoke Assumption~\ref{assum:M_appendix}, which guarantees:
$$
M(p(\rvx_\rvb|\rvz)) = M(p(\rvx_\rvb|h'(\rvz))) = h(\rvz)
$$
where is $h(\rvz)=\hat{\rvz}$ iff $h$ is differentiable.. 
Hence, we conclude the subspace identifiability of $\rvz$.

\subsection{Proof of Theorem~2}\label{sec:approxiden}




We reorganize the Theorem~2 below:

{\it\bfsection{Theorem 2}
Consider observed variables $\rvx \in \mathbb{R}^K$ and the estimated latent variables $\hat{\rvz}\in\mathbb{R}^N$, suppose 
the following assumptions hold:
\begin{enumerate}[i]
    \item\label{assum:bound} The joint distribution of $(\rvx, \rvz)$ admits a bounded density with respect to a suitable product measure defined on their supports. Furthermore, all marginal and conditional densities derived from this joint distribution are also bounded;
    \item\label{assum:subiden_ijc_supp} The operators $L_{\rvx_\rva|\rvz}$ and $L_{\rvx_\rva|\rvx_\rvc}$ are injective and bounded-below;
    \item\label{assum:subiden_car_supp} $L_{\rvx_\rva,\rvx_\rvb|\rvx_\rvc}L^{-1}_{\rvx_\rva|\rvx_\rvc}$ has distinct spectral values with cardinality equal to that of $\Lambda_{\rvx_\rvb|\rvz}$. Also,  Let $\Theta:\mathcal{Z}\rightarrow \mathcal{D}, s.t. p(\rvx_\rvb|\rvz)=\Theta(\rvz)$ denotes the map from $\mathcal{Z}$ to $\mathcal{D}$. $\Theta$ is injective;
    \item\label{assum:lipschitz} $\mathcal Z$ is compact, and $\Theta$ in Assumption~\ref{assum:subiden_car_supp} is continuously differentiable. There exists $c_{\Theta}>0$ such that, for an admissible
    invertible transformation $h$, $\Theta$ satisfies the lower bound of Bi-Lipschitz continuity by:
    \begin{align}
    c_{\Theta}
    |\hat{\mathbf{z}}-h(\mathbf{z})|
    \leq
    \|\Theta(\mathbf{z})
    -\widehat{\Theta}(\hat{\mathbf{z}})\|_{\mathcal F}.
    \label{eq:inverse_lipschitz_thm2}
    \end{align}
    where $\mathcal F$ denotes the function norm.
\end{enumerate}
Then we can obtain $\varepsilon \leq \frac{\kappa\overline{Per} + \kappa\widehat{\overline{Per}}}{c_\Theta}$.
}


\bfsection{Proof:}


Following Steps 1 and 2 in the proof of Theorem~\ref{thm:submanifold_iden}, for each fixed $\rvx_\rvb$, the multiplication operator $\Lambda_{\rvx_\rvb\mid \rvz}$ is given by
\begin{align}
(\Lambda_{\rvx_\rvb\mid \rvz}f)(\rvz)=p(\rvx_\rvb\mid \rvz)f(\rvz).
\nonumber
\end{align}
Since $\Theta$ is continuous, $\forall\rvz\in\mathcal Z, p(\rvx_\rvb|\rvz)\in ess\, ran_{\rvz} p(\rvx_\rvb|\rvz)=\rho_\Lambda$, where $ess\, ran$ denotes the essential range.

Given $dist(\rho_\Lambda , \rho)\leq \kappa\overline{Per}$ in Eq.~\ref{eq:bound}, we can define a similar bound by $dist(\hat{\rho}_\Lambda , \rho)\leq \hat{\kappa}\widehat{\overline{Per}}$. Therefore, there always exist $i$, such that:
\begin{equation}
    |\rho^i_\Lambda-\rho^i_{\hat{\Lambda}}| =|\rho^i_\Lambda - \rho^i + \rho^i -\rho^i_{\hat{\Lambda}}|
    \leq |\rho^i_\Lambda - \rho^i| + |\rho^i_{\hat\Lambda}-\rho^i| 
    \leq \kappa\overline{Per} + \hat{\kappa}\widehat{\overline{Per}}
\end{equation}\label{eq:bound_rho}
    
where the first $\leq$ is because of the triangle inequality. Eq.~\ref{eq:bound_rho} suggests that, for almost all $\rvz,\hat{\rvz}$:
\begin{align}
    |p^i(\rvx_\rvb|\rvz) - p^i(\rvx_\rvb|\hat{\rvz})|\leq \kappa\overline{Per} + \hat{\kappa}\widehat{\overline{Per}}
\end{align}

Therefore, Taking the essential supremum over $\rvx_\rvb$ gives
\begin{align}
\|\Theta(\rvz)-\widehat{\Theta}(\hat \rvz)\|_{\mathcal F}
&=
ess\,sup_{\rvx_\rvb}
|p^i(\rvx_\rvb\mid \rvz)-p^i(\rvx_\rvb\mid\hat \rvz)|
\nonumber\\
&\leq
\kappa\overline{\mathrm{Per}}+\hat{\kappa}\widehat{\overline{\mathrm{Per}}}.
\label{eq:distribution_bound}
\end{align}

By Assumption~\ref{assum:lipschitz},
\begin{align}
c_{\Theta}|\hat \rvz-h(\rvz)|\leq\|\Theta(\rvz)-\widehat{\Theta}(\hat \rvz)\|_{\mathcal F}
\end{align}
Combining this inequality with Eq.~\eqref{eq:distribution_bound}
yields
\begin{align} \varepsilon=\inf\limits_{h\in\mathcal H}|\hat \rvz-h(\rvz)|\leq\frac{\kappa\overline{\mathrm{Per}}+\hat{\kappa}\widehat{\overline{\mathrm{Per}}}}{c_{\Theta}}
\nonumber
\end{align}
This completes the proof.

\subsection{Discussions on $\overline{Per}$}\label{sec:bound_per}

In this secion, we discuss the bound of $\overline{Per}$ since it plays a critical role in Theorem 1. 
We present the following Lemma with the assistance of Assumption 3:
\begin{lemma}\label{lemma:Per}
    Let $\|\cdot\|$ be the operator norm. Consider an operator $L(\cdot)$ subject to the distance operator $\rvd$ in Eq.~\ref{eq:error_oper_1}. If $\|L^{-1}_{\rvx_\rva|\rvx_\rvc}\rvd\|< 1$, then $(L_{\rvx_\rva|\rvx_\rvc}-\rvd)^{-1}\leq L^{-1}(\rvx_\rva|\rvx_\rvc)+\frac{||L^{-1}(\rvx_\rva|\rvx_\rvc)||^2||\rvd||}{1-||L^{-1}(\rvx_\rva|\rvx_\rvc)\times \rvd||}$
\end{lemma}

\bfsection{Proof:}
Let $(L_{\rvx_\rva|\rvx_\rvc}-\rvd)^{-1}=L^{-1}_{\rvx_\rva|\rvx_\rvc}+B$, we can formalize:
\begin{align}\label{eq:per_derivation}
    L_{\rvx_\rva|\rvz}L_{\rvx_\rvb|\rvz}L^{-1}_{\rvx_\rva|\rvz}
    & = (L_{\rvx_\rva,\rvx_\rvb|\rvx_\rvc} - \rvd(\rvx_\rvb))(L_{\rvx_\rva|\rvx_\rvc} - \rvd)^{-1} \nonumber\\
    & = (L_{\rvx_\rva,\rvx_\rvb|\rvx_\rvc} - \rvd(\rvx_\rvb))(L^{-1}_{\rvx_\rva|\rvx_\rvc}+B) \nonumber\\
    & = L_{\rvx_\rva,\rvx_\rvb|\rvx_\rvc}L^{-1}_{\rvx_\rva|\rvx_\rvc} +
    \underbrace{L_{\rvx_\rva,\rvx_\rvb|\rvx_\rvc}B-\rvd(\rvx_\rvb)L^{-1}_{\rvx_\rva|\rvx_\rvc}-\rvd(\rvx_\rvb)B}_{\it Per}
\end{align}
where we can bound $B$ as follows:
\begin{align}\label{eq:bound_operator}
    B & = (L_{\rvx_\rva|\rvx_\rvc}-\rvd)^{-1} - L^{-1}_{\rvx_\rva|\rvx_\rvc} \nonumber\\
    & = (L_{\rvx_\rva|\rvx_\rvc}(I-L^{-1}_{\rvx_\rva|\rvx_\rvc}\rvd))^{-1} - L^{-1}_{\rvx_\rva|\rvx_\rvc}\nonumber\\
    & = (I-L^{-1}_{\rvx_\rva|\rvx_\rvc}\rvd)^{-1}L^{-1}_{\rvx_\rva|\rvx_\rvc} - L^{-1}_{\rvx_\rva|\rvx_\rvc}\nonumber\\
    & \Rightarrow ||B||\leq ||L^{-1}_{\rvx_\rva|\rvx_\rvc}||\times||(I-L^{-1}_{\rvx_\rva|\rvx_\rvc}\rvd)^{-1}-I||\nonumber\\
    & \Rightarrow ||B||\leq ||L^{-1}_{\rvx_\rva|\rvx_\rvc}||(||(I-L^{-1}_{\rvx_\rva|\rvx_\rvc}\rvd)^{-1}||+||-I||)\nonumber\\
    & \Rightarrow ||B||\leq ||L^{-1}_{\rvx_\rva|\rvx_\rvc}||(\frac{1}{1-||L^{-1}_{\rvx_\rva|\rvx_\rvc}\rvd||}-1) 
    \nonumber\\
    & \Rightarrow ||B||\leq ||L^{-1}_{\rvx_\rva|\rvx_\rvc}||(\frac{||L^{-1}_{\rvx_\rva|\rvx_\rvc}\rvd||}{1-||L^{-1}_{\rvx_\rva|\rvx_\rvc}\rvd||})\nonumber\\
    & \Rightarrow ||B||\leq\frac{||L^{-1}_{\rvx_\rva|\rvx_\rvc}||^2||\rvd||}{1-||L^{-1}_{\rvx_\rva|\rvx_\rvc}\rvd||}
\end{align}
Substitute the result from Eq.~\ref{eq:bound_operator} to Eq.~\ref{eq:error_derivation}, we can obtain that 
\begin{align}\label{eq:boundderive_per}
    \|\it Per\|\leq\|L_{\rvx_\rva,\rvx_\rvb|\rvx_\rvc}\frac{||L^{-1}_{\rvx_\rva|\rvx_\rvc}||^2||\rvd||}{1-||L^{-1}_{\rvx_\rva|\rvx_\rvc}\rvd||}-\rvd(\rvx_\rvb)L^{-1}_{\rvx_\rva|\rvx_\rvc}-\rvd(\rvx_\rvb)\frac{||L^{-1}_{\rvx_\rva|\rvx_\rvc}||^2||\rvd||}{1-||L^{-1}_{\rvx_\rva|\rvx_\rvc}\rvd||}=\overline{Per}\
\end{align}
We can adopt the similar strategy to derive that, if $||L^{-1}_{\rvx_\rva|\rvx_\rvc}\hat{\rvd}||<1$:
\begin{align}\label{eq:boundderive_hatper}
    \|\it \widehat{Per}\|\leq\|L_{\rvx_\rva,\rvx_\rvb|\rvx_\rvc}\frac{||L^{-1}_{\rvx_\rva|\rvx_\rvc}||^2||\hat{\rvd}||}{1-||L^{-1}_{\rvx_\rva|\rvx_\rvc}\hat{\rvd}||}-\hat{\rvd}(\rvx_\rvb)L^{-1}_{\rvx_\rva|\rvx_\rvc}-\hat{\rvd}(\rvx_\rvb)\frac{||L^{-1}_{\rvx_\rva|\rvx_\rvc}||^2||\hat{\rvd}||}{1-||L^{-1}_{\rvx_\rva|\rvx_\rvc}\hat{\rvd}||}=\widehat{\overline{Per}}
\end{align}

Upon Eq.~\ref{eq:boundderive_per}, we can derive the sufficient condition to safisfy $||Per||_{op}\leq \overline{Per}<\eta$, which is critical for precise identifiability in Theorem~\ref{thm:submanifold_iden}:

\begin{lemma}\label{lemma:sufficient_eta}
    suppose that
    \begin{align}
        &||L^{-1}_{\rvx_\rva|\rvx_\rvc}||_{op}||\rvd||_{op}<\eta \nonumber\\
        &||\rvd||_{op}<\frac{\eta-||L^{-1}_{\rvx_\rva|\rvx_\rvc}||_{op}||\rvd||_{op}}{||L^{-1}_{\rvx_\rva|\rvx_\rvc}||_{op}^2||\rvd||(||L_{\rvx_\rva,\rvx_\rvb|\rvx_\rvc}||_{op}+||\rvd||_{op})+||L^{-1}_{\rvx_\rva|\rvx_\rvc}||_{op}(\eta-||L^{-1}_{\rvx_\rva|\rvx_\rvc}||_{op}||\rvd||_{op})}
    \end{align}
    Then we can obtain $||Per||_{op}\leq \overline{Per}<\eta$.
\end{lemma}

\subsection{Proof of Proposition 2}\label{sec:proof_thm2}

{\it\bfsection{Proposition 2}
Suppose Theorem~\ref{thm:submanifold_iden} holds, and the following assumptions and regularization conditions are satisfied:
\begin{enumerate}
    \item\label{assum:independence_z} {\bf (Independence)}: For $n\in [N]$, $\rvz^n$ are independent of each other: $p(\rvz)=\prod_{n=1}^N p(\rvz^n)$. 
    \item\label{assum:ssparsity} {\bf (Structural sparsity)}: The map $\Phi: \mathcal{Z}\rightarrow \mathcal P$ is smooth, where $\mathcal{Z}$ denotes the support of $\rvz$, and $\mathcal{P}$ denotes the space of PDF $p(\rvx|\rvz)$ ( $\mathcal P \subset \mathcal{L}^2(\mathcal{X},\mu)$ ($\mu$ denotes the measure)). Let $G$ and $\hat{G}$ be the binary support matrix of the Jacobian $J_{\Phi}(\rvz)$ and $J_{\hat{\Phi}}(\hat{\rvz})$, respectively, i.e. $G^{i,j}= 1$ if $||\frac{\partial p(\rvx^i_\rvb|\rvz)}{\partial \rvz^j}||_{\mathcal{L}^2}\neq 0$, and $0$ for otherwise.
    For each $n \in [1,N]$, there exists a subset of indices $\mathcal{C}_k$ satisfying $\bigcap_{m\in\mathcal{C}_k}G^{m,:}=\{n\}$;
\end{enumerate}
Then, $\hat{\rvz}$ must correspond component-wise to a permutation of the true latent variables $\rvz$.
}

\bfsection{Proof:} 
Theorem~\ref{thm:submanifold_iden} guarantees the existence of an invertible trasnformation $h$ such that $\hat{\rvz}=h(\rvz)$, 
Our goal is to show that $h$ is a composition of a permutation and component-wise diagonal transformations.

Applying the chain-rule by taking Fr\'echet derivatives to $\Phi(\rvz) = \hat{\Phi}(\hat{\rvz})$, we can obtain:
\begin{align}
    \Phi(\rvz) & = \hat{\Phi}(\hat{\rvz}) \nonumber\\
    \Rightarrow J_{\Phi}(\rvz) & = J_{\hat{\Phi}} (\hat{\rvz})J_h(\rvz) \label{eq:equivalience}
\end{align}

Assumption~\ref{assum:subiden_car_supp} in Theorem~\ref{thm:submanifold_iden} assume the injective of $\Phi$. 
Therefore, Eq.~\ref{eq:equivalience} implies
\begin{align}\label{eq:permutation}
    \forall \{i,j\}\in\mathcal{G}, \{i,\sigma(j)\}\hat{\mathcal{G}}
\end{align}
where $\sigma$ denotes the permutation transformation in this section.

For simplicity, we denote denote $J_h$ by $\mathbf{H}$. Suppose, for contradiction, that $H(\rvz)$ is not a composition of a diagonal matrix and a permutation matrix, i.e., there exist $j_1 \neq j_2$ such that:
\begin{equation}
    \operatorname{supp}(\mathbf{H}^{j_1,:}) \cap \operatorname{supp}(\mathbf{H}^{j_2,:}) \neq \emptyset.
\end{equation}
Let $j_3$ be an element in this intersection, so $\sigma(j_3) \in \operatorname{supp}(\mathbf{H}^{j_1,:}) \cap \operatorname{supp}(\mathbf{H}^{j_2,:})$. Without loss of generality, assume $j_3 \neq j_1$. According to Assumption \ref{assum:ssparsity}, there exists a set $\mathcal{C}_{j_1}$ containing $j_1$ such that:
\begin{equation}
    \bigcap_{i \in \mathcal{C}_{j_1}} G^{i,:} = \{ j_1 \}.
\end{equation}
Since $j_3 \neq j_1$, it must be that:
\begin{equation}
    j_3 \notin \bigcap_{i \in \mathcal{C}_{j_1}} G^{i,:},
\end{equation}
implying there exists some $i_3 \in \mathcal{C}_{j_1}$ such that:
\begin{equation} \label{eq:uc_contradiction_t5}
    j_3 \notin G^{i_3,:}.
\end{equation}
However, since $j_1 \in G^{i_3,:}$, we have $(i_3, j_1) \in G$. Given $\mathbf{H}$ is a permuted transformation, Eq.~\ref{eq:permutation} implies:
\begin{equation}
    (i_3, \sigma(j_3)) \in \hat{G}.
\end{equation}
However, this means $(i_3, j_3) \in G$, which contradicts Eq. \ref{eq:uc_contradiction_t5}. This contradiction implies $\mathbf{H}$ must be a composition of a permutation matrix and a diagonal matrix.

Together with the equation $J_{g} = J_{\hat{g}} \mathbf{H}$, we achieve the desired result that $t$ is composed of a permutation and component-wise invertible functions.

\subsection{Additional Proposition}\label{sec:proof_thm3}

Assumption~\ref{assum:independence_z} in Proposition~\ref{thm:variabe_iden} claims the independence between each dimension of $\rvz$.
To relax such a constraint, we can alternatively allow latent variables $\rvz$ to exhibit dependence via a known auxiliary domain variable $\rvu$. 
In other words, we assume conditional independence across dimensions of $\rvz$ given $\rvu$, i.e., $p(\rvz|\rvu)=\prod_{n=1}^N p(\rvz^n|\rvu)$.
Specifically, we modify the original data-generating process in Eq.~\ref{eq:dgp} to incorporate the domain index $\rvu$ explicitly:
\begin{align}\label{eq:dgp_continuous_domain}
    \rvx = g(\rvz,\epsilon), \quad
    \epsilon = e(\rvz,\rvu,\eta)
\end{align}
where $\rvu$ denotes the domain index.
Under these conditions, we establish the following proposition:
\begin{proposition}\label{thm:varchange_iden}
    Suppose 
the subspace identifiability condition in Theorem~\ref{thm:submanifold_iden} holds. Additionally, assume the following conditions:
\begin{enumerate}
    \item\label{assum:ind_domains} {\bf (Conditional independence)}: Latent variables are conditionally independent given domain $\rvu$: $p(\rvz|\rvu)=\prod_{n=1}^N p(\rvz^n|\rvu)$. Also, $p(\rvz|\rvu)$ is positive and twice differentiable.
    \item\label{assum:variability_domains} {\bf (Sufficient variability)}: There exist $2N+1$ distinct domain values $\rvu \in [1, 2N+1]$, such that the $2N$ vectors $\rvw(\rvz,\rvu)-\rvw(\rvz,\rvu_0)$ (with $\rvu \neq \rvu_0$) are linearly independent, where the vector $\rvw(\rvz, \rvu)$ is defined as:
    $\rvw(\rvz,\rvu)=\{\rvv(\rvz,\rvu),\rvv'(\rvz,\rvu)\}$
    with
    \begin{align*}
        \rvv(\rvz,\rvu)&=\left(\frac{\partial \log p(\rvz^1|\rvu)}{\partial \rvz^1}, \dots, \frac{\partial \log p(\rvz^N|\rvu)}{\partial \rvz^N}\right) \\
        \rvv'(\rvz,\rvu)&=\left(\frac{\partial^2 \log p(\rvz^1|\rvu)}{(\partial \rvz^1)^2}, \dots, \frac{\partial^2 \log p(\rvz^N|\rvu)}{(\partial \rvz^N)^2}\right)
    \end{align*}
\end{enumerate}
Then $\{\hat{\rvz}^{\hat{n}}|\hat{n}\in [1,N]\}$ must be a component-wise transformation of a permuted version of true $\{\rvz^n|n\in [1,N]\}$
\end{proposition}

\begin{examplebox}
\bfsection{Insight:}
Proposition~\ref{thm:varchange_iden} shows the identifiability results on non-i.i.d. scenario. Domain-dependent changes in $p(\rvz\mid\rvu)$ provide sufficient variation to distinguish the latent coordinates, while Theorem~\ref{thm:submanifold_iden} first removes the ambiguity introduced by the misspecified observation structure and $\epsilon$. 
Thus, non-i.i.d. variation can replace marginal latent independence as the additional information needed to reduce the remaining invertible ambiguity to component-wise transformations and permutation.
\end{examplebox}

\bfsection{Proof:} 
By Theorem~\ref{thm:submanifold_iden} there exists an invertible reparameterization $h:\mathcal{Z}\to\mathcal{Z}$ such that $\hat{\rvz}=h(\rvz)$ and $\rvz=h^{-1}(\hat{\rvz})$. Applying the change-of-variables formula to the conditional densities (for any fixed $\rvu$) gives:
\begin{equation}\label{eq:change_of_variables}
p_{\hat{\rvz}\mid\rvu}(\hat{\rvz}\mid\rvu)
=
p_{\rvz\mid\rvu}\!\big(h^{-1}(\hat{\rvz})\mid\rvu\big)
\bigl|\det J_{h^{-1}}(\hat{\rvz})\bigr|.
\end{equation}
Taking logarithms yields
\begin{equation}\label{eq:log_densities}
\log p_{\hat{\rvz}\mid\rvu}(\hat{\rvz}\mid\rvu)
=
\log p_{\rvz\mid\rvu}(\rvz\mid\rvu)
+
\log\bigl|\det J_{h^{-1}}(\hat{\rvz})\bigr|,
\end{equation}
Under Assumption~\ref{assum:ind_domains}, we have
\begin{equation}\label{eq:log_densities_expanded}
\sum_{i=1}^n \log p_{\hat{\rvz}^i\mid\rvu}(\hat{\rvz}^i\mid\rvu)
=
\sum_{i=1}^n \log p_{\rvz^i\mid\rvu}(\rvz^i\mid\rvu)
+
\log\bigl|\det J{_{h^{-1}}}(\hat{\rvz})\bigr|.
\end{equation}
We then take second derivatives with respect to $\hat{\rvz}^k$ and $\hat{\rvz}^v$ for $k\neq v$. Since each term on the left-hand side of Eq.~\ref{eq:log_densities_expanded} depends only on a single coordinate $\hat{\rvz}^i$, we have $\partial \log p_{\hat{\rvz}^i\mid\rvu}(\hat{\rvz}^i\mid\rvu)/\partial \hat{\rvz}^k=0$ for $i\neq k$, which implies
\begin{equation}\label{eq:left_hand_side}
\frac{\partial^2}{\partial \hat{\rvz}^k\partial \hat{\rvz}^v}
\sum_{i=1}^{n}\log p_{\hat{\rvz}^i\mid\rvu}(\hat{\rvz}^i\mid\rvu)=0.
\end{equation}
For the right-hand side, define for $i=1,\ldots,n$
\begin{equation}
    \tilde h^{i,(k)}:=\frac{\partial \rvz^i}{\partial \hat{\rvz}^k}, 
{\tilde{h'}}^{i,(k,v)}:=
\frac{\partial^2 \rvz^i}{\partial \hat{\rvz}^k\partial \hat{\rvz}^v}
\end{equation}
\begin{equation}
    \eta'_i(\rvz^i,\rvu):=\frac{\partial}{\partial \rvz^i}\log p_{\rvz^i\mid\rvu}(\rvz^i\mid\rvu), 
    \eta''_i(\rvz^i,\rvu):=\frac{\partial^2}{\partial {\rvz^i}^2}\log p_{\rvz^i\mid\rvu}(\rvz^i\mid\rvu)
\end{equation}
A direct application of the chain rule gives
\begin{equation}
\sum_{i=1}^{n}\Bigl(\eta''_i(\rvz^i,\rvu)\tilde h^{i,(k)}\tilde h^{i,(v)}
+\eta'_{i}
(\rvz^i,\rvu){\tilde{h'}^{i,(k,v)}}
\Bigr)
+\frac{\partial^2}{\partial \hat{\rvz}^k\partial \hat{\rvz}^v}\log\bigl|\det J_{h^{-1}}(\hat{\rvz})\bigr|=0.
\end{equation}
Fix $(k,v)$ with $k\neq v$ and evaluate this identity at $2n\!+\!1$ distinct values of the conditioning variable, $\rvu^{(j)}$ for $j\in\{0,1,\ldots,2n\}$. Subtracting the equation at $\rvu^{(0)}$ from that at $\rvu^{(j)}$ cancels the log-determinant term (which does not depend on $\rvu$) and yields, for $j=1,\ldots,2n$,
\begin{equation}\label{eq:linear_system}
\sum_{i=1}^{n}\Bigl(\bigl[\eta''_i(\rvz^i,\rvu^{(j)})-\eta''_i(\rvz^i,\rvu^{(0)})\bigr]\tilde h^{i,(k)}\tilde h^{i,(v)}
+\bigl[\eta'_i(\rvz^i,\rvu^{(j)})-\eta'_i(\rvz^i,\rvu^{(0)})\bigr]{\tilde{h'}^{i,(k,v)}}\Bigr)=0.
\end{equation}
Let
\begin{equation}
\rvw(\rvz,\rvu):=\bigl(\eta''_1(\rvz^1,\rvu),\ldots,\eta''_n(\rvz^n,\rvu),\eta'_1(\rvz^1,\rvu),\ldots,\eta'_n(\rvz^n,\rvu)\bigr)^\top.
\end{equation}
Under Assumption~\ref{assum:variability_domains}, the $2n$ vectors $\rvw(\rvz,\rvu^{(j)})-\rvw(\rvz,\rvu^{(0)})$ for $j=1,\ldots,2n$ are linearly independent, so the only solution to the homogeneous linear system Eq.~\ref{eq:linear_system} is
\begin{equation}
\tilde h^{i,(k)}\tilde h^{i,(v)}=0\quad\text{and}\quad {\tilde{h'}^{i,(k,v)}}=0
\qquad\text{for all } i\in\{1,\ldots,n\}\text{ and all } k\neq v.
\end{equation}
Hence each row of the Jacobian $J_{h^{-1}}(\hat{\rvz})$ has at most one nonzero entry, and all mixed second derivatives vanish. Since $h^{-1}$ is invertible, each row must in fact have exactly one nonzero entry; moreover, two distinct rows cannot share the same nonzero column (otherwise $\det J_{h^{-1}}(\hat{\rvz})=0$), so there exists a permutation $\pi$ such that
\begin{equation}
\hat{\rvz}^{\pi(i)} = h^i(\rvz^i)\qquad\text{for } i=1,\ldots,n,
\end{equation}
which shows that $\hat{\rvz}$ is obtained from $\rvz$ by a permutation of component-wise invertible transformations.

\begin{examplebox}\bfsection{Assumption Intuitions:}

$\bullet$ Assumption~\ref{assum:ind_domains} relaxes the marginal independence of $\rvz$, which allows that $\rvz$ can be statistically dependent.

$\bullet$ Assumption~\ref{assum:variability_domains} requires the domain variable to induce sufficiently rich changes in the first- and second-order derivatives of the latent conditional log-densities. These variations provide enough constraints to distinguish different latent variable.

\end{examplebox}

\section{Synthetic Experiment Details}\label{sec:impt}

\subsection{Synthetic data-generating process}\label{sec:synthetic}

We construct a synthetic data-generating process that realizes the
latent-dependent noise and misspecified conditional structure considered
in this work.

Let
\begin{align}
\rvz
&=
(\rvz^1,\rvz^2,\rvz^3)^\top,
\qquad
\rvz
\sim
\mathcal N(0,I_3),
\end{align}
so that the latent coordinates are mutually independent.
We independently sample
\begin{align}
\eta^0,\eta^\rva,\eta^\rvb,\eta^\rvc
\stackrel{\mathrm{ind.}}{\sim}
\mathcal N(0,I_3),
\end{align}
where $\eta^0$ is shared across the three observed blocks and
$\eta^\rva,\eta^\rvb,\eta^\rvc$ are block-specific exogenous noises.

For each $v\in\{\rva,\rvb,\rvc\}$, define
\begin{align}
\epsilon^v
&=
C^v\rvz+\tau\eta^0+s_v\eta^v,
\qquad
\tau>0,\quad s_v>0,
\nonumber\\
\rvx_v
&=
\Psi(B^v\rvz+\epsilon^v)
=
\Psi(M^v\rvz+\tau\eta^0+s_v\eta^v),
\qquad
M^v:=B^v+C^v,
\label{eq:synthetic_dgp}
\end{align}
where where the existence of $\eta^0$ guarantees the dependence of $\rvx^v$ conditioning on $\rvz$, $\Psi$ is applied coordinate-wise and is defined by
\begin{align}
\Psi(t)
=
t+\alpha\tanh(t),
\qquad
0<\alpha<1.
\end{align}
Since
\begin{align}
\Psi'(t)
=
1+\alpha\operatorname{sech}^2(t)
\in[1,1+\alpha],
\end{align}
$\Psi$ is a smooth bijection with a smooth inverse.

We choose
\begin{align}
M^\rva
=
M^\rvc
=
I_3,
\qquad
M^\rvb
=
\begin{pmatrix}
1&r&0\\
r&0&1\\
0&r&1
\end{pmatrix},
\qquad
0<r<1.
\label{eq:synthetic_M}
\end{align}
Since
\begin{align}
\det(M^\rvb)
=
-r(1+r)\neq0,
\end{align}
$M^\rvb$ is invertible.
We further set
\begin{align}
B^v=C^v=\frac{1}{2}M^v,
\end{align}
so that
\begin{align}
\epsilon^v
=
\frac{1}{2}M^v\rvz
+\tau\eta^0+s_v\eta^v.
\end{align}
Hence the complete noise variable
$\epsilon=(\epsilon^\rva,\epsilon^\rvb,\epsilon^\rvc)$ explicitly
depends on $\rvz$, as required by Eq.~\eqref{eq:dgp}.
The resulting observation
\begin{align}
\rvx
=
(\rvx_\rva,\rvx_\rvb,\rvx_\rvc)
\in\mathbb R^9
\end{align}
contains three disjoint three-dimensional blocks.

In practice, we sample We sample 10,000 points from the data synthesizing process above mentioned. We leverage 90\% among them for training, and the rest 10\% for test.

\subsection{Assumption Justifications:}\label{sec:assumjust}

Below we justify if all assumptions from Theorem~\ref{thm:submanifold_iden} and Proposition~\ref{thm:variabe_iden} hold for our data synthesizing:

\paragraph{Misspecified conditional structure.}
Conditioning on $\rvz$ fixes all deterministic terms involving
$\rvz$, while the three blocks continue to share $\eta^0$.
Let
\begin{align}
\rvy_v
:=
M^v\rvz+\tau\eta^0+s_v\eta^v.
\end{align}
Then, for distinct $v,v'\in\{\rva,\rvb,\rvc\}$,
\begin{align}
\operatorname{Cov}(\rvy_v,\rvy_{v'}\mid\rvz)
=
\tau^2I_3.
\end{align}
Therefore, for every $\tau>0$,
\begin{align}
p(\rvx_\rva,\rvx_\rvb,\rvx_\rvc\mid\rvz)
\neq
p(\rvx_\rva\mid\rvz)
p(\rvx_\rvb\mid\rvz)
p(\rvx_\rvc\mid\rvz).
\end{align}
Because $\Psi$ is invertible, the conditional dependence is preserved
after the transformation. The conditionally independent setting is
recovered when $\tau=0$.

\paragraph{Constructing both the precise and approximate identifiability regime.}
The shared-noise strength $\tau$ directly controls the amount of
structural misspecification. Let
$\rvd_\tau(\rvx_\rvb)$, $\rvd_\tau$, and
$\mathit{Per}_\tau$ denote the corresponding difference and
perturbation operators generated with $\tau$ in Eq.~\ref{eq:synthetic_dgp}.


For the Gaussian construction in Eq.~\ref{eq:synthetic_dgp}, the
conditional covariance induced by the shared disturbance is
$\tau^2I_3$. Hence the corresponding conditional densities vary
smoothly with $\tau^2$. On the operator spaces considered here, there
exist finite constants $C_d,C_b>0$ and let $\tau_{\max}>0$ such that, for
$0<\tau\leq\tau_{\max}$,
\begin{align}
\|\rvd_\tau\|_{\mathrm{op}}
&<
C_d\tau^2,
\nonumber\\
\operatorname*{ess\,sup}_{\rvx_\rvb}
\|\rvd_\tau(\rvx_\rvb)\|_{\mathrm{op}}
&<
C_b\tau^2.
\label{eq:d_tau_bound}
\end{align}

Let
\begin{align}
M_0
&:=
\sup_{0<\tau\leq\tau_{\max}}
\|
L^{-1}_{\rvx_\rva|\rvx_\rvc,\tau}
\|_{\mathrm{op}},
\nonumber\\
M_1
&:=
\sup_{0<\tau\leq\tau_{\max}}
\|
L_{\rvx_\rva,\rvx_\rvb|\rvx_\rvc,\tau}
\|_{\mathrm{op}},
\end{align}
which are finite under Assumptions~\ref{assum:error} and
\ref{assum:injective}.
Using Lemma~\ref{lemma:Per}, whenever
$M_0C_d\tau^2<1$,
\begin{align}
\|\mathit{Per}_\tau\|_{\mathrm{op}}
\leq
&
\left(
M_1+C_b\tau^2
\right)
\frac{
M_0^2C_d\tau^2
}{
1-M_0C_d\tau^2
}
+
M_0C_b\tau^2
\nonumber\\
=:&
\overline{\mathit{Per}}.
\label{eq:per_tau_bound}
\end{align}


Let
\begin{align}
\eta
=
\frac{\delta}{2\kappa}-\alpha
>0
\end{align}
be defined as in Section~\ref{sec:preciseiden}.
By continuity of
$\overline{\mathit{Per}}(\tau)$ at $\tau=0$, there exists
$\tau_0>0$ such that
\begin{align}
0<\tau<\tau_0
\quad\Longrightarrow\quad
0<
\overline{\mathit{Per}}(\tau)
<
\eta.
\label{eq:construct_per_eta}
\end{align}
We choose the shared-noise strength
$\tau\in(0,\tau_0)$ in the precise-identifiability experiments.
Thus, the synthetic data simultaneously exhibit structural
misspecification ($\tau>0$) and satisfy the perturbation condition
required by Proposition~\ref{prop:eigenvalue}.

To evaluate approximate identifiability, we use the same
data-generating process but increase $\tau$ to generate stronger
structural misspecification. We select values of $\tau$ for which
\begin{align}
\overline{\mathit{Per}}(\tau)
\geq
\eta,
\label{eq:approx_tau}
\end{align}
so that the sufficient perturbation condition for precise
identifiability is no longer satisfied. All remaining assumptions of
Theorem~\ref{thm:approx_vareps} are kept unchanged. Therefore, the same
synthetic model produces both regimes: small nonzero $\tau$ yields
precise identifiability, whereas larger perturbations are used to
evaluate the approximate-identifiability guarantee.

\paragraph{Bounded densities.}
Before applying $\Psi$, the joint vector
$(\rvz,\rvy_\rva,\rvy_\rvb,\rvy_\rvc)$ is a nondegenerate multivariate
Gaussian because $s_v>0$ for every block. Hence its joint, marginal,
and conditional densities are bounded.
Since $\Psi'$ is uniformly bounded above and bounded away from zero,
the change-of-variables formula preserves boundedness of the
corresponding densities after transformation.
Thus the bounded-density regularity condition in
Assumption~\ref{assum:error} is satisfied.

\paragraph{Injectivity of $\Theta$.}
For the $\rvx_\rvb$ block,
\begin{align}
\Psi^{-1}(\rvx_\rvb)\mid\rvz
\sim
\mathcal N
\left(
M^\rvb\rvz,
(\tau^2+s_\rvb^2)I_3
\right).
\label{eq:xb_conditional}
\end{align}
Suppose
\begin{align}
p(\cdot\mid\rvz)
=
p(\cdot\mid\rvz').
\end{align}
Equation~\eqref{eq:xb_conditional} then implies
\begin{align}
M^\rvb\rvz
=
M^\rvb\rvz'.
\end{align}
Since $M^\rvb$ is invertible,
\begin{align}
\rvz=\rvz'.
\end{align}
Therefore the conditional-distribution map
\begin{align}
\Theta:\mathcal Z\rightarrow\mathcal D,
\qquad
\Theta(\rvz)
=
p(\rvx_\rvb\mid\rvz),
\end{align}
is injective.



\paragraph{Structural sparsity for Proposition~\ref{thm:variabe_iden}.}
For each observed coordinate $\rvx^j$, define
\begin{align}
\Phi^j(\rvz)
:=
p(\rvx^j\mid\rvz).
\end{align}
Since the conditional mean of each pre-transformed coordinate depends
on $\rvz$ through the corresponding row of $M^v$, the support of
$\partial\Phi^j/\partial\rvz^i$ coincides with the support of that row.
Ordering the observations as
\begin{align}
(
\rvx_\rva^1,\rvx_\rva^2,\rvx_\rva^3,
\rvx_\rvb^1,\rvx_\rvb^2,\rvx_\rvb^3,
\rvx_\rvc^1,\rvx_\rvc^2,\rvx_\rvc^3
),
\end{align}
the structural-support matrix is
\begin{align}
G
=
\begin{pmatrix}
1&0&0\\
0&1&0\\
0&0&1\\
1&1&0\\
1&0&1\\
0&1&1\\
1&0&0\\
0&1&0\\
0&0&1
\end{pmatrix}.
\label{eq:synthetic_G}
\end{align}
Thus the structural-sparsity condition of
Proposition~\ref{thm:variabe_iden} is satisfied.
Furthermore,
\begin{align}
p(\rvz)
=
\prod_{i=1}^{3}p(\rvz^i)
\end{align}
by construction, satisfying its latent-independence condition.

\subsection{Implementations \& Training Details.}\label{sec:synexp_imt}
In our synthetic experiments, we set the dimensions $N = 3$ and $K = 9$. The encoders, decoder, and normalizing flow modules were each implemented using single-layer multilayer perceptrons (MLPs) followed by Leaky ReLU activations.

Our implementation utilized PyTorch 1.11.0. For optimization, we adopted the AdamW optimizer \cite{adamw_iclr19}, which is known for enhancing generalization in deep learning models. The hyperparameters were configured as follows: a learning rate of $1\times 10^{-3}$ and a batch size of 64. To guarantee robustness and statistical reliability, each model was trained using 10 different random seeds. We report the overall performance as the mean $\pm$ standard deviation computed across these runs.
The loss function employed balances the reconstruction error and the KL-divergence, with weighting coefficients set to $\beta_1 = \beta_2 = 0.02$.
All experiments were performed on a single NVIDIA GeForce RTX 2080 Ti GPU equipped with 11GB of memory.


\subsection{The objective of IndVAE}\label{sec:obj_indvae}

In this section, we explain our trained objective for IndVAE, which is designed by taking inspirations from \citep{hu2008instrumental}
To incorporate with the conditional independence assumption, we consider the $k$-th observed variable $\rvx^k$ is generated as $\rvx^k=g^k(\rvz,\epsilon)$. Accordingly, the log-likelihood of the data generating process of Eq.~\ref{eq:dgp} can be transformed as follows:
\begin{align}
    \log p(\rvz,\epsilon,\rvx) & =
    \log p_\theta(\rvx|\rvz,\epsilon)+\log p_\gamma(\epsilon|\rvz)+ \log p_\delta(\rvz) \nonumber\\
    & = \sum\limits_{k=1}^{K}\log p_\theta(\rvx^k|\rvz,\epsilon)+\log p_\gamma(\epsilon|\rvz)+ \log p_\delta(\rvz)
\end{align}

Accordingly, the loss function becomes:
\begin{align}\label{eq:elbo_indvae}
\mathcal{L}_{\text{ELBO}} = &
\mathbb{E}_{\hat{\rvz} \sim q_{\psi}, \hat{\epsilon} \sim q_{\phi}}\left[\sum\limits_{k=1}^{K}\log p_{\theta}(\hat{\rvx}^k|\hat{\rvz},\hat{\epsilon})\right]\nonumber\\
& -\beta_1\mathbb{E}_{\hat{\rvz}\sim q_{\psi}} \big(\log q\left(\hat{\rvz}\vert \rvx \right) - \log p_\delta(\rvz)\big)
-\beta_2\mathbb{E}_{\hat{\rvz}\sim q_{\psi},\hat{\epsilon} \sim q_{\phi}} \big(\log q\left(\hat{\epsilon}\vert \rvx \right) - \log p_\gamma(\hat{\epsilon}|\hat{\rvz})\big)
\end{align}

\section{Additional Experiments}\label{sec:exp_u}

\input{tabs/tab_impelement}

\subsection{Real-world Experiment for Justifying Proposition~\ref{thm:variabe_iden}}

\bfsection{Task Setup:}
To validate our proposed identifiability theories in realistic and complex scenarios, we apply them to the task of Person Classification. 
In Person Classification, the goal is to assign a unique identity index to each individual, based on input images. 
This setup aligns well with our generalized dependency structure setting, as each image of an individual inherently contains noise, such as varying poses, gaits, or clothes, making it challenging to disentangle these factors from the underlying identity. 
Also, different body parts cannot be independent conditioning upon the latent person identify index.
Consequently, this task serves as a solid playground for evaluating the robustness and efficacy of our theoretical framework in addressing real-world complexities.

In our implementation, we first employ a pretrained feature extractor to derive feature representations of each individual person, denoted as $\rvx \in \mathbb{R}^K$. We consider $\rvx$ are generated by the latent variable $\rvz \in \mathbb{R}^N$, directly associated with each person's identity index, along with a dependent noise variable $\epsilon \in \mathbb{R}^M$, capturing variations such as pose, gait, or clothes.
Inspired by the two-phase training pipeline proposed by \citep{li2024learning,li2025identification}, we adapt our approach to Person Classification task as follows. Figure~\ref{fig:dp_pc} suggests the data generating process of the Person Classification task.
First, we train our approach by optimizing the objective function detailed in Eq.~\ref{eq:elbo}. Subsequently, we introduce a classifier $\hat{c}$, implemented by a multilayer perceptron (MLP), to predict the one-hot encoded index label $\hat{y}$ from the inferred latent representation $\hat{\rvz}$:
$\hat{y} = \text{MLP}(\hat{\rvz})$
The classifier is optimized using a cross-entropy loss given by:
$\mathcal{L}^{\text{CE}}_{\text{cls}} = -\mathbb{E}_{\hat{y}}\left[\text{one-hot}(y)\cdot\log(\text{softmax}(\hat{y}))\right]$
where $\text{one-hot}(y)$ denotes the one-hot embedding of the true person index label.

\input{figs/fig_2step}

\bfsection{Data and Comparing Approaches}
We conduct our experiments on the MSMT17 dataset~\citep{wei2018person}, which comprises images of 4,101 unique individuals. Each individual in the dataset has more than 10 images, resulting in a total of over 120,000 images. We partition the dataset into three parts: $60\%$ for training, $20\%$ for validation, and the remaining $20\%$ for test.

For performance comparison, we select several state-of-the-art methods on the task of Person Classification,including GTL~\citep{yang2025crossmodalfewshotlearninggenerative}, AGW~\citep{ye2021deep}, TransReID~\citep{he2021transreid}, and CLIPReID~\citep{li2023clip}.
We also benchmark against MCRL \citep{sun2025causal} and IndVAE \citep{hu2008instrumental} to evaluate the efficacy of our identifiability results under generalized dependency structure.

\input{tabs/acc}

\bfsection{Results \& Discussions:}
Table~\ref{table:sota} reports the comparison of Top-1 Accuracy (Acc) among state-of-the-art methods on the MSMT17 dataset. 
Our method achieves a superior performance, substantially surpassing approaches that do not explicitly handle dependent noise, such as IndVAE \citep{hu2008instrumental} and MCRL~\citep{sun2025causal}. 
Furthermore, our proposed method demonstrates notable improvements over leading methods including GTL~\citep{yang2025crossmodalfewshotlearninggenerative}, CLIPReID~\citep{li2023clip}, TransReID~\citep{he2021transreid}, and AGW~\citep{ye2021deep}, outperforming them by significant margins.
These results clearly highlight the effectiveness and robustness of our proposed framework in accurately addressing generalized dependency structures in complex real-world scenarios.

\bfsection{Implementations:}

To obtain a fair comparison, we adopt the approach outlined by \citep{yang2025crossmodalfewshotlearninggenerative}, employing the pretrained CLIP model \citep{clip_icml21} as the visual encoder to generate 1280-dimensional representations for $\rvx$. Table~\ref{tab:architecture} summarizes the specific network architectures implemented for our experiments on the real-world MSTM17 dataset.

To train our framework, we utilize the AdamW optimizer combined with a cosine annealing learning rate schedule. The initial learning rate is set to $2 \times 10^{-3}$, with a weight decay parameter of $1 \times 10^{-2}$ to prevent overfitting. The ELBO loss function incorporates equal weighting coefficients $\beta_1 = \beta_2 = 0.02$. We use a batch size of 128, chosen to balance computational efficiency with optimization stability. The framework is implemented in PyTorch. Training is done for 80 epochs on a multi-GPU configuration comprising four NVIDIA GeForce RTX 2080 Ti GPUs, collectively providing 44GB of memory.

\subsection{Additional Ablation Studies}

\bfsection{Ablations for higher dimensional $z$}

we additionally evaluated the scalability of our method to higher latent dimensions on the synthetic dataset by varying the latent dimensionality $N \in \{8, 12, 18\}$, while keeping the network architecture, training protocol, and all other hyperparameters fixed.
Table~\ref{tab:highdim_synth} reports the MCC and compares with IndVAE~\cite{hu2008instrumental}:
\begin{table}[h]
\small
\centering
\begin{tabular}{ccc}
\toprule
$N$ & IndVAE & Ours \\
\midrule
$8$  & $0.64 \pm 0.06$ & $0.80 \pm 0.02$ \\
$12$ & $0.51 \pm 0.03$ & $0.68 \pm 0.05$ \\
$18$ & $0.47 \pm 0.04$ & $0.61 \pm 0.05$ \\
\bottomrule
\end{tabular}
\caption{MCC on the synthetic dataset for increasing latent dimensionality $N$.}
\label{tab:highdim_synth}
\end{table}
Even as the latent dimension increases, our method consistently achieves higher MCC than IndVAE, indicating that the Jacobian-based sparsity regularization remains effective and that our approach scales well to higher-dimensional latent spaces within the considered regime.

\bfsection{Ablations for hyperparameter sensitivity}

In our implementation on the synthetic dataset, we set the weight of the sparsity regularizer to $1$ and the KL weights to $\beta_1=\beta_2=0.02$.
In this section, we conducted a sensitivity analysis in which we vary \emph{one} of these three hyperparameters at a time while keeping the others fixed at their default values.
For all runs we use the same network architecture, batch size, number of epochs, learning rate, and training protocol as in the main experiments.
The numbers reported below are MCC scores (mean $\pm$ std.\ over multiple runs) on the synthetic dataset.
\begin{table}[t]
\centering
\caption{Sensitivity of MCC to regularization hyperparameters.}
\setlength{\tabcolsep}{6pt}
\begin{tabular}{l c c}
\hline\hline
Hyperparameter & Value & MCC \\
\hline
{}{}{$\lambda$} 
 & $1$    & $0.87 \pm 0.04$ \\
 & $0.1$  & $0.82 \pm 0.03$ \\
 & $0.01$ & $0.70 \pm 0.07$ \\
 & $10$   & $0.68 \pm 0.05$ \\
\hline
{}{}{$\beta_1$}
 & $0.02$  & $0.87 \pm 0.04$ \\
 & $1$     & $0.73 \pm 0.02$ \\
 & $0.001$ & $0.65 \pm 0.04$ \\
\hline
{}{}{$\beta_2$}
 & $0.02$  & $0.87 \pm 0.04$ \\
 & $1$     & $0.82 \pm 0.01$ \\
 & $0.001$ & $0.78 \pm 0.06$ \\
\hline\hline
\end{tabular}
\label{tab:hyperparam_sensitivity}
\end{table}
These results show 
the default setting $(\lambda, \beta_1, \beta_2) = (1, 0.02, 0.02)$ obtains the best results.

\bfsection{Additional real-world experiments}

we expand our real-world evaluation beyond the original dataset and consider two additional person identity classification benchmarks.
Specifically, we use the \textsc{SYSU-MM01} dataset~\cite{wu2017rgb} and the \textsc{RobotPKU} dataset~\cite{liu2017online}.
We only use the RGB modality, since our focus is not on cross-modal person re-identification.
\textsc{SYSU-MM01} contains RGB images of $491$ identities from $6$ cameras, with a total of $30{,}071$ images.
\textsc{RobotPKU} contains more than $16{,}000$ RGB images of $180$ identities, captured under dynamic robotic viewpoints.
These datasets thus provide solid playgrounds for our experiments.

For performance comparison, we follow the same Person Classification protocol as in our main experiment and compare against several state-of-the-art methods, including GTL~\citep{yang2025crossmodalfewshotlearninggenerative}, AGW~\citep{ye2021deep}, TransReID~\citep{he2021transreid}, CLIPReID~\citep{li2023clip}, LDP-net~\cite{zhou2023revisiting}, Style~\cite{fu2023styleadv}, as well as MCRL~\citep{sun2025causal} and IndVAE~\citep{hu2008instrumental}.
Tables~\ref{tab:robotpku} and~\ref{tab:sysu} report the Top-1 classification accuracy (mean $\pm$ std.\ over multiple runs).

\begin{table}[h]
\centering
\caption{Comparison of Top-1 Accuracy on the \textsc{RobotPKU} dataset.}
\setlength{\tabcolsep}{6pt}
\begin{tabular}{l | c }
\hline\hline
Methods & Acc \\
\hline
AGW~\citep{ye2021deep} & $87.6 \pm 0.8$ \\
TransReID~\citep{he2021transreid} & $90.2 \pm 0.9$ \\
CLIPReID~\citep{li2023clip} & $91.7 \pm 1.1$ \\
GTL~\cite{yang2025crossmodalfewshotlearninggenerative} & $93.9 \pm 0.5$ \\
\hline\hline
MCRL~\citep{sun2025causal} & $94.5 \pm 1.0$ \\
IndVAE~\citep{hu2008instrumental} & \underline{$95.8 \pm 0.8$} \\
Ours & {\bf $98.4 \pm 0.5$} \\
\hline
\end{tabular}
\label{tab:robotpku}
\end{table}

\begin{table}[h]
\centering
\caption{Comparison of Top-1 Accuracy on the \textsc{SYSU-MM01} dataset.}
\setlength{\tabcolsep}{6pt}
\begin{tabular}{l | c }
\hline\hline
Methods & Acc \\
\hline
LDP-net~\cite{zhou2023revisiting} & $91.7 \pm 1.1$ \\
Style~\cite{fu2023styleadv} & $92.8 \pm 0.8$ \\
CLIPReID~\citep{li2023clip} & $94.1 \pm 1.0$ \\
GTL~\cite{yang2025crossmodalfewshotlearninggenerative} & $95.7 \pm 0.4$ \\
\hline\hline
MCRL~\citep{sun2025causal} & $96.4 \pm 0.8$ \\
IndVAE~\citep{hu2008instrumental} & \underline{$96.8 \pm 0.5$} \\
Ours & {\bf $99.1 \pm 0.5$} \\
\hline
\end{tabular}
\label{tab:sysu}
\end{table}

To further examine the effect of the architecture choice, we replace the MLP in our framework on the MSTM-17 dataset with a single-layer Gated Recurrent Unit (GRU) \cite{goodfellow2016deep} using the same hidden dimension (the classifier architecture and all training and evaluation protocols remain unchanged, and we set $N=M=32$ for fair comparison).
The resulting Top-1 accuracies are:
GRU: $96.9 \pm 0.4$ versus our original MLP-based model: $96.2 \pm 0.6$.
The GRU improves the classification results slightly against the MLP architecture.
Overall, our method consistently outperforms strong baselines across three real-world person identity datasets.

\subsection{Additional Real-world Experiments for Proposition~\ref{thm:varchange_iden}}

\begin{figure}
\centering
\includegraphics[width=0.2\linewidth]{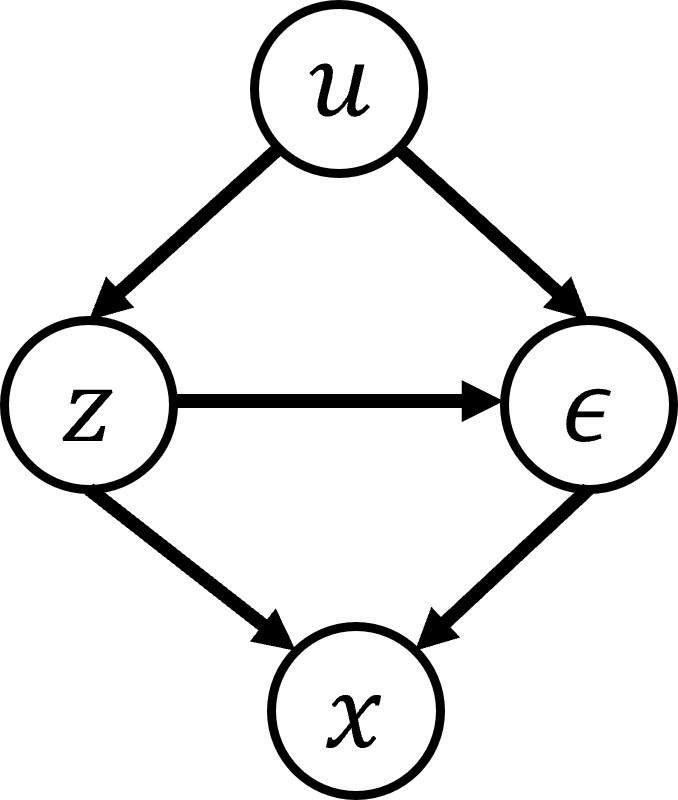}
\caption{
Visualization of the data generations of Eq.~\ref{eq:dgp_continuous_domain}.
}
\label{fig:dp_u}
\end{figure}

\bfsection{Training Objective:}
Figure~\ref{fig:dp_u} visualizes the data-generating process described in Eq.~\ref{eq:dgp_continuous_domain}. Accordingly, the likelihood for this process, given the known auxiliary variable $\rvu$, is expressed as:
\begin{align}\label{eq:likelihood_u}
p(\rvz,\epsilon,\rvx | \rvu) =  p_\theta(\rvx|\rvz,\epsilon)p_\gamma(\epsilon|\rvz,\rvu)p_\delta(\rvz|\rvu)
\end{align}

As a result, we redesign the encoder and prior module to learn the distribution $p_\gamma(\epsilon|\rvz,\rvu)$ and $p_\delta(\rvz|\rvu)$, as shown in Eq.~\ref{eq:likelihood_u}, while keeping the decoder unchanged. 
Accordingly, the ELBO is:
\begin{align}\label{eq:elbo_u}
\mathcal{L}_{\text{ELBO}} = &
\mathbb{E}_{\hat{\rvz} \sim q_{\psi}, \hat{\epsilon} \sim q_{\phi}}\left[\log p_{\theta}(\hat{\rvx}^k|\hat{\rvz},\hat{\epsilon}^k)\right] \nonumber\\
& -\beta_1\mathbb{E}_{\hat{\rvz}\sim q_{\psi}} \big(\log q\left(\hat{\rvz}\vert \rvx \right) - \log p_\delta(\rvz|\rvu)\big)\nonumber \\
& -\beta_2\mathbb{E}_{\hat{\rvz}\sim q_{\psi},\hat{\epsilon} \sim q_{\phi}} \big(\log q\left(\hat{\epsilon}\vert \rvx \right) - \log p_\gamma(\hat{\epsilon}|\hat{\rvz}, \rvu)\big)
\end{align}

We further evaluate our method on two standard domain adaptation
benchmarks, Office-Home \citep{venkateswara2017deep} and VisDA-2017 \citep{peng2018visda}. As shown in
Table~\ref{tab:officehome}, our method achieves the best performance on
all $12$ transfer directions of Office-Home. In particular, it improves
over UniMoS by $4.6$ percentage points on $A\rightarrow C$ and $3.3$
points on $A\rightarrow P$. Averaged over all transfer directions, our
method achieves an accuracy of $80.17\%$, compared with $77.88\%$ for
the strongest competing method, UniMoS.

Table~\ref{tab:visda} shows a similar trend on VisDA-2017. Our method
achieves the best performance on $10$ of the $12$ categories and ties
for the best result on Motorcycle. Notable improvements are observed on
Person ($87.2\%$), Bus ($93.7\%$), and Truck ($67.5\%$). The average
accuracy reaches $90.46\%$, outperforming UniMoS at $88.05\%$. The only
category on which our method does not achieve the best result is Car,
where DAPrompt obtains $77.9\%$ compared with our $76.8\%$.

Overall, the consistent gains across different source--target pairs and
object categories indicate that the learned representation transfers
more effectively across domains. Together with the person-index
classification results, these findings are consistent with our
hypothesis that explicitly separating the latent representation $\rvz$
from the dependent noise $\epsilon$ improves not only in-domain
discrimination, but also representation robustness and transferability.



\begin{table*}[t]
\centering
\caption{Results on Office-Home.}
\label{tab:officehome}
\resizebox{\textwidth}{!}{
\begin{tabular}{lcccccccccccc}
\toprule
Method
& $A\rightarrow C$
& $A\rightarrow P$
& $A\rightarrow R$
& $C\rightarrow A$
& $C\rightarrow P$
& $C\rightarrow R$
& $P\rightarrow A$
& $P\rightarrow C$
& $P\rightarrow R$
& $R\rightarrow A$
& $R\rightarrow C$
& $R\rightarrow P$ \\
\midrule
CLIP \citep{radford2021learning}
& 51.7 & 81.5 & 82.3 & 71.7 & 81.5 & 82.3
& 71.7 & 51.7 & 82.3 & 71.7 & 51.7 & 81.5 \\
DAPrompt \citep{ge2023domain}
& 54.1 & 84.3 & 84.8 & 74.4 & 83.7 & 85.0
& 74.5 & 54.6 & 84.8 & 75.2 & 54.7 & 83.8 \\
ADCLIP \citep{singha2023adclip}
& 55.4 & 85.2 & 85.6 & 76.1 & 85.8 & 86.2
& 76.7 & 56.1 & 85.4 & 76.8 & 56.1 & 85.5 \\
UniMoS \citep{li2024split}
& 59.5 & 89.4 & 86.9 & 75.2 & 89.6 & 86.8
& 75.4 & 58.4 & 87.2 & 76.9 & 59.5 & 89.7 \\
\textbf{Ours}
& \textbf{64.1} & \textbf{92.7} & \textbf{88.5}
& \textbf{79.2} & \textbf{91.0} & \textbf{89.5}
& \textbf{77.9} & \textbf{60.2} & \textbf{89.6}
& \textbf{78.4} & \textbf{59.9} & \textbf{91.0} \\
\bottomrule
\end{tabular}
}
\end{table*}

\begin{table*}[t]
\centering
\caption{Results on VisDA-2017.}
\label{tab:visda}
\resizebox{\textwidth}{!}{
\begin{tabular}{lcccccccccccc}
\toprule
Method
& Plane
& Bicycle
& Bus
& Car
& Horse
& Knife
& Motorcycle
& Person
& Plant
& Skateboard
& Train
& Truck \\
\midrule
CLIP \citep{radford2021learning}
& 98.2 & 83.9 & 90.5 & 73.5 & 97.2 & 84.0
& 95.3 & 65.7 & 79.4 & 89.9 & 91.8 & 63.3 \\
DAPrompt \citep{ge2023domain}
& 97.8 & 83.1 & 88.8 & 77.9 & 97.4 & 91.5
& 94.2 & 79.7 & 88.6 & 89.3 & 92.5 & 62.0 \\
ADCLIP \citep{singha2023adclip}
& 98.1 & 83.6 & 91.2 & 76.6 & 98.1 & 93.4
& 96.0 & 81.4 & 86.4 & 91.5 & 92.1 & 64.2 \\
UniMoS \citep{li2024split}
& 97.7 & 88.2 & 90.1 & 74.6 & 96.8 & 95.8
& 92.4 & 84.1 & 90.8 & 89.0 & 91.8 & 65.3 \\
\textbf{Ours}
& \textbf{98.5} & \textbf{90.1} & \textbf{93.7}
& \textbf{76.8} & \textbf{98.7} & \textbf{97.4}
& \textbf{96.0} & \textbf{87.2} & \textbf{92.0}
& \textbf{93.4} & \textbf{94.2} & \textbf{67.5} \\
\bottomrule
\end{tabular}
}
\end{table*}

\section{Conclusion \& Limitations}

We study identifiability under a misspecified structure where observations may remain dependent given the latent variables and
the noise may itself depend on the latents. We characterize structural misspecification as a perturbed factor analysis problem. Under controlled perturbation and spectral separation, we establish precise subspace identifiability and further obtain component-wise identifiability under structural sparsity. 
Beyond this precise regime, we establish approximate subspace identifiability by bounding the identification error of the recovered representation from its admissible equivalence class. We further develop an unsupervised variational estimator and evaluate the resulting framework on synthetic and real-world tasks.

A limitation of the current theory is that the approximate guarantee is restricted to the subspace level. Extending it to approximate component-wise identifiability requires additional information to understand the component-level ambiguity under perturbation, beyond the conditions used for approximate subspace recovery. Characterizing the minimal structural assumptions or auxiliary information sufficient for such a guarantee is an important direction for future work.

%% file: tabs/tab_impelement.tex
\begin{table}[h]
\small
\caption{The details of our network architectures for the experiment on MSTM-17 dataset, where BS means batch size, $N=32$, $M=32$ and $K=1280$.}
\label{tab:architecture}
\begin{center}
\begin{tabular}{lll}
\multicolumn{1}{c}{Configuration}  
&\multicolumn{1}{c}{Description}
&\multicolumn{1}{c}{Output dimensions}
\\\hline \\
Encoder   $q\psi$    &   & \\
Input: $\rvx$  &  & $\text{BS}\times K$ \\
Dense &  128 neurons, LeakyReLU & $\text{BS}\times 128$\\
Dense &  128 neurons, LeakyReLU & $\text{BS}\times 128$\\
Dense &  Output embeddings & $\text{BS}\times 2N$\\
Bottleneck &  Compute mean and variance of posterior & $\mu_\rvz$, $\sigma_\rvz$\\
Reparameterization &  Sequential sampling & $\hat{\rvz}$\\
\hline \\
Encoder   $q_\phi$    &   & \\
Input: $\rvx$  &  & $\text{BS}\times K$ \\
Dense &  128 neurons, LeakyReLU & $\text{BS}\times 128$\\
Dense &  128 neurons, LeakyReLU & $\text{BS}\times 128$\\
Dense &  Output embeddings & $\text{BS}\times 2M$\\
Bottleneck &  Compute mean and variance of posterior & $\mu_\epsilon$, $\sigma_\epsilon$\\
Reparameterization &  Sequential sampling & $\hat{\epsilon}$\\
\hline \\
Decoder        &  \\
Input: $\hat{\rvz},\hat{\epsilon}$ &   & $\text{BS}\times (N+M)$\\
Dense &  128 neurons, LeakyReLU & $\text{BS}\times  128$\\
Dense &  128 neurons, LeakyReLU & $\text{BS}\times  128$\\
Dense &  input embeddings & $\text{BS}\times  K$\\
\hline\\ 
Prior module      &  \\
Input &  $\hat{\rvz},\hat{\epsilon}$ & $\text{BS}\times  (N+M)$\\
InverseTransformation &  $\hat{\eta}$ & $\text{BS}\times  M$\\
JacobianCompute &  $\log |\text{det}J_{\hat{e}}|$ & $\text{BS}$\\
\hline \\
Classifier       &  \\
Input: $\hat{\rvz}$ &  & $\text{BS}\times  N$\\
Dense &  256 neurons, LeakyReLU & $\text{BS}\times  256$\\
Dense &  256 neurons, LeakyReLU & $\text{BS}\times  256$\\
Dense &  output one-hot embeddings & $\text{BS}\times 4101$\\
\hline
\hline
\end{tabular}
\end{center}
\end{table}

%% file: figs/fig_2step.tex
\begin{wrapfigure}{r}{3cm}
\vspace{-10px}
\includegraphics[width=\linewidth]{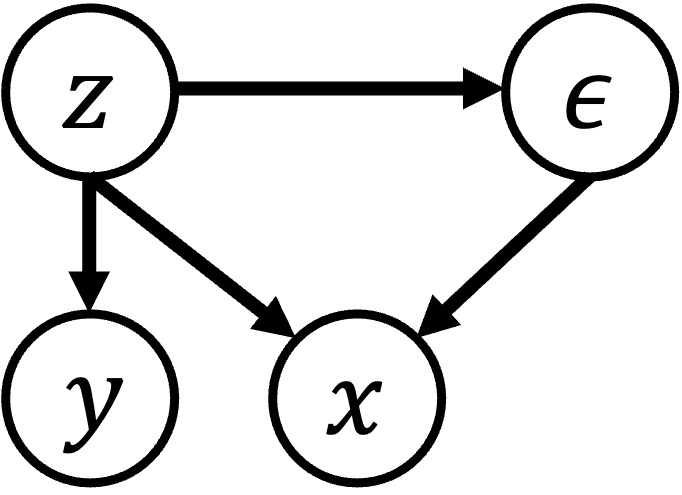}
\caption{
Visualization of the data generating process for Person Classification task
}
\label{fig:dp_pc}
\end{wrapfigure}

%% file: tabs/acc.tex
\begin{table}[h]
\caption{Comparison of Top-1 Accuracy on MSMT17 dataset}


\begin{center}
\begin{tabular}{l | c }
\hline
\hline
\text{Methods} 
&\multicolumn{1}{c}{Acc} \\
\hline
 AGW \citep{ye2021deep} & 85.5 $\pm$ 1.2 \\
 TransReID \citep{he2021transreid} & 87.8 $\pm$ 0.5\\
 CLIPReID \citep{li2023clip} & 90.1 $\pm$ 0.3\\
 GTL \cite{yang2025crossmodalfewshotlearninggenerative} & 91.5 $\pm$ 1.2 \\
 \hline 
 \hline
 MCRL \citep{sun2025causal} & 92.6 $\pm$ 0.9 \\
IndVAE \citep{hu2008instrumental} & \underline{93.1 $\pm$ 0.5}  \\
Ours  & \bf{96.2 $\pm$ 0.6}  \\

\hline
\end{tabular}
\end{center}
\label{table:sota}
\end{table}